\documentclass[twoside,11pt]{article}

\usepackage{blindtext}

\usepackage[preprint]{jmlr2e}
\usepackage{booktabs}
\usepackage{amsmath}
\usepackage{algpseudocode, algorithm}
\usepackage{float}
\usepackage{subcaption}
\usepackage{tabularx}
\usepackage{siunitx}

\newcolumntype{L}{>{\raggedright\arraybackslash}X}

\newenvironment{theorem*}[1][]{%
    \par\medskip\noindent      
    \textbf{Theorem}           
    \if\relax\detokenize{#1}\relax\else\space(#1)\fi 
    \textbf{.} \itshape        
}{%
    \par\medskip               
}

\usepackage{lastpage}
\jmlrheading{23}{2025}{1-\pageref{LastPage}}{1/21; Revised 5/22}{9/22}{21-0000}{Enzo Spotorno and Josafat Ribeiro}

\ShortHeadings{Hard-Constrained NNs with Physics-Embedded Architecture in CPSs}{Enzo Spotorno and Josafat Ribeiro}
\firstpageno{1}

\begin{document}

\title{Hard-Constrained Neural Networks with Physics-Embedded Architecture for Residual Dynamics Learning and Invariant Enforcement in Cyber-Physical Systems}

\author{\name Enzo Nicolás Spotorno \email enzoniko@lisha.ufsc.br \\
       \addr Department of Informatics and Statistics\\
       Federal University of Santa Catarina\\
       Santa Catarina, 88040-900, Brazil
       \AND
       \name Josafat Leal Filho \email josafat@lisha.ufsc.br \\
       \addr Department of Informatics and Statistics\\
       Federal University of Santa Catarina\\
       Santa Catarina, 88040-900, Brazil
       \AND
       \name Antônio Augusto Fröhlich \email guto@lisha.ufsc.br \\
       \addr Department of Informatics and Statistics\\
       Federal University of Santa Catarina\\
       Santa Catarina, 88040-900, Brazil}

\editor{TO FILL}

\maketitle

\begin{abstract}
This paper presents a framework for physics-informed learning in complex cyber-physical systems governed by differential equations with both unknown dynamics and algebraic invariants. First, we formalize the Hybrid Recurrent Physics-Informed Neural Network (HRPINN), a general-purpose architecture that embeds known physics as a hard structural constraint within a recurrent integrator to learn only residual dynamics. Second, we introduce the Projected HRPINN (PHRPINN), a novel extension that integrates a predict–project mechanism to strictly enforce algebraic invariants by design. The framework is supported by a theoretical analysis of its representational capacity. We validate HRPINN on a real-world battery prognostics DAE and evaluate PHRPINN on a suite of standard constrained benchmarks. The results demonstrate the framework's potential for achieving high accuracy and data efficiency, while also highlighting critical trade-offs between physical consistency, computational cost, and numerical stability, providing practical guidance for its deployment.
\end{abstract}


\section{Introduction}
\label{sec:introduction}

The increasing complexity and autonomy of Cyber-Physical Systems (CPS) present a fundamental modeling challenge: how to create simulation models that are both highly accurate and physically consistent. These models must capture the behavior of systems governed by a mix of well-understood physical laws and complex, unmodeled dynamics such as friction, wear, and thermal effects. This dual requirement is critical in safety-conscious applications like autonomous vehicles and industrial robotics, where predictive accuracy alone is insufficient \citep{van2022executable}. Models must also be \textbf{physically consistent} (respecting fundamental laws to ensure their outputs are meaningful and safe for decision-making) and \textbf{data-efficient}, learning robustly from the limited and often expensive data available from real-world operations \citep{van2022executable}. A key application driving this need is the development of high-fidelity \textbf{digital twins}, which are becoming critical for real-time prediction, control, and safety assessment \citep{eckhart2019digital, mihai2022digital}. 

We tackle this challenge for a specific class of physical systems: those whose behavior can be described by known ordinary differential equations (ODEs) augmented by unknown residual dynamics and, when present, algebraic invariants.

This dual requirement creates a modeling tension. Traditional mechanistic (first-principles) models for such systems guarantee physical consistency but are often incomplete \citep{karniadakis2021physics, raissi2019physics}: they may omit frictional effects, wear and degradation, thermal couplings, or other parasitic dynamics that materially affect behavior. Conversely, purely data-driven models can capture these nuanced residual dynamics \citep{karniadakis2021physics, raissi2019physics} but typically provide no built-in guarantee of physical plausibility, which can lead to long-horizon violations of basic principles (e.g., conservation laws or algebraic invariants) and unsafe extrapolations. Physics-Informed Machine Learning (PIML) seeks to bridge this gap by incorporating prior physical knowledge into learning systems, improving sample efficiency and interpretability while constraining the hypothesis class \citep{karniadakis2021physics} (the set of all possible functions the model is allowed to learn—forcing it to only consider solutions that are physically plausible).

A commonly used pragmatic strategy within PIML is grey-box (residual) learning: explicitly encode trusted, well-understood physics and learn only the unknown residual dynamics. Doing so prevents the learning component from wasting capacity trying to re-learn established physical laws that were built and refined over many years; instead, the network focuses representational power on modeling discrepancies that are difficult to derive analytically \citep{somers2023digital}. This residual/grey-box approach underlies several recent methods, including Universal/Neural ODE variants (UDEs/NODEs) that augment known dynamics with learned terms. While efficient at modeling unknown dynamics, these hybrid residual models do not, by construction, enforce algebraic invariants and can struggle in stiff regimes \citep{chen2018neural, rackauckas2020universal, legaard2023constructing}.

While the residual learning approach is promising, enforcing physical consistency remains a central challenge in PIML, with existing strategies forcing difficult trade-offs. \textbf{Soft-constrained} methods, like Physics-Informed Neural Networks (PINNs), incorporate physics via loss penalties but offer no strict guarantees at inference time. Conversely, \textbf{hard-constraint} approaches either impose strong architectural priors (e.g., Hamiltonian Neural Networks \citep{greydanus2019hamiltonian}) that limit their applicability to non-ideal, real-world systems, or use projection steps that are typically paired with inefficient, black-box dynamics models (e.g. \citep{white2023stabilized, pal2025semi}). Taken together, these observations reveal a gap: no single existing approach simultaneously (i) embeds known ODE structure as a hard architectural constraint, (ii) learns only unknown residual dynamics to preserve representational efficiency, and (iii) enforces algebraic invariants exactly within an end-to-end trainable recurrent framework.

This paper proposes a unified framework designed to close that gap by formalizing  Hybrid Recurrent Physics-Informed Neural Networks (HRPINNs) and introducing their projected extension, PHRPINN. The core architecture hard-codes known governing equations inside a recurrent integrator cell to learn only residual dynamics. For systems with algebraic invariants, the predict–project mechanism in PHRPINN strictly enforces these constraints by design. Our main contributions are threefold:

\begin{enumerate}
    \item We \textbf{formalize} the HRPINN architecture as a general-purpose method for residual dynamics learning in time-series and demonstrate its effectiveness for partially observed systems.
    \item We \textbf{introduce} PHRPINN, a novel predict–project extension that strictly enforces algebraic invariants, and we analyze the trade-offs between its fast and robust projection variants.
    \item We provide both \textbf{theoretical and empirical evidence} for the framework's benefits, including a proof of representational equivalence (Theorem~\ref{thm:equivalence}) and experimental results demonstrating improved data efficiency, physical consistency, and optimization stability.
\end{enumerate}

In support of reproducibility, we release our implementation, hyperparameter tables, and ablation scripts. This work focuses on index-1 constrained systems, with extensions to higher-index DAEs and more exhaustive benchmarking left to future work.

The remainder of this paper is organized as follows. Section~\ref{sec:related_work} situates our work within the existing literature on physics-informed machine learning. Section~\ref{sec:background} establishes the mathematical formulation for the dynamical systems we address. Our proposed HRPINN and PHRPINN architectures are detailed in Section~\ref{sec:method_overview}, including their theoretical underpinnings. Section~\ref{sec:casestudies} describes the two case studies used for validation: a real-world battery prognostics task and a suite of standard constrained benchmark systems. The experimental setup and results for these case studies are presented in Section~\ref{sec:experiments_and_results}. We discuss the practical implications and deployment considerations of our findings in Section~\ref{sec:discussion}. Finally, we outline the limitations and avenues for future research in Section~\ref{sec:limitations} and conclude the paper in Section~\ref{sec:conclusion}.

\section{Related Work}
\label{sec:related_work}

We position HRPINN and PHRPINN relative to established paradigms in physics-informed machine learning (PIML). The choice of a PIML methodology represents a critical trade-off between model expressiveness, physical rigidity, computational cost, and the rigor of constraint satisfaction \citep{karniadakis2021physics}. This section deconstructs this landscape to build a clear argument for a framework that synthesizes the strengths of prior approaches (namely data efficiency and hard-constraint enforcement) while mitigating their core limitations. Table~\ref{tab:taxonomy} provides a high-level summary of the methods analyzed.

\begin{table}[ht!]
\centering
\caption{A Comparative Taxonomy of PIML Method Families.}
\label{tab:taxonomy}
\renewcommand{\arraystretch}{1.25}
\small
\begin{tabular}{@{}p{4.2cm}p{9.8cm}@{}}
\toprule
\multicolumn{2}{l}{\textbf{Soft PINNs}}\\
\midrule
\textbf{Enforcement} & Training (soft penalty) \\
\textbf{What is Learned} & Full dynamics ($f_{\text{true}}$) \\
\textbf{Inference Guarantee} & None \\
\textbf{For Dissipative Sys.} & High \\
\textbf{Primary Advantage} & Flexibility \\
\textbf{Primary Limitation} & Gradient pathologies \\
\midrule
\multicolumn{2}{l}{\textbf{HNNs / LNNs}}\\
\midrule
\textbf{Enforcement} & Architectural \\
\textbf{What is Learned} & Scalar potential ($H, L$) \\
\textbf{Inference Guarantee} & Conserves learned $H, L$ \\
\textbf{For Dissipative Sys.} & Low (needs extensions) \\
\textbf{Primary Advantage} & Long-term conservation \\
\textbf{Primary Limitation} & Inflexible, poor for non-ideal systems \\
\midrule
\multicolumn{2}{l}{\textbf{PNODEs (Black-Box)}}\\
\midrule
\textbf{Enforcement} & Inference (projection) \\
\textbf{What is Learned} & Full dynamics ($f_{\text{true}}$) \\
\textbf{Inference Guarantee} & State is on manifold \\
\textbf{For Dissipative Sys.} & High \\
\textbf{Primary Advantage} & Hard constraint enforcement \\
\textbf{Primary Limitation} & High sample complexity \\
\midrule
\multicolumn{2}{l}{\textbf{ALM / OptLayer}}\\
\midrule
\textbf{Enforcement} & Training (semi-hard) \\
\textbf{What is Learned} & Full dynamics ($f_{\text{true}}$) \\
\textbf{Inference Guarantee} & None \\
\textbf{For Dissipative Sys.} & High \\
\textbf{Primary Advantage} & Better than simple penalty \\
\textbf{Primary Limitation} & Training-time only \\
\midrule
\multicolumn{2}{l}{\textbf{HRPINN / PHRPINN (Ours)}}\\
\midrule
\textbf{Enforcement} & Architectural \& projection \\
\textbf{What is Learned} & Residual dynamics ($f_{\text{unk}}$) \\
\textbf{Inference Guarantee} & Known physics enforced; state on manifold (PHRPINN) \\
\textbf{For Dissipative Sys.} & High \\
\textbf{Primary Advantage} & Data efficiency \& guarantees \\
\textbf{Primary Limitation} & Needs known model part \\
\bottomrule
\end{tabular}
\end{table}

\subsection{Open Challenges in Soft Constraints: Why Penalties Are Not Enough}
The most common PIML strategy is the Physics-Informed Neural Network (PINN), which adds governing equations as soft penalty terms to the loss function \citep{raissi2019physics}. While flexible, this reliance on regularization presents two fundamental flaws for safety-critical systems. First, it creates significant optimization challenges, as the delicate balancing of competing loss terms often leads to "gradient flow pathologies" that stall training or require complex, problem-specific tuning \citep{wang2021understanding, karniadakis2021physics} (e.g., gradient pulling the network towards a noisy data point may be directly cancelled out by the physics gradient demanding a smooth solution, causing the optimizer to stall). Second, and more critically, soft constraints provide \textbf{no algorithmic guarantee of physical consistency} at inference time. The model may learn to satisfy the physics on average during training, but it remains susceptible to long-horizon drift where small, compounding errors lead to physically implausible or unsafe predictions \citep{pal2025semi}. This inherent unreliability of penalty-based methods directly motivates the need for \textit{hard-constrained} architectures.

\subsection{The Dilemma of Hard Constraints: Rigidity vs. Inefficiency}
In response to the failings of soft penalties, hard-constraint methods guarantee compliance by design. However, they introduce a new trade-off between architectural rigidity and data inefficiency.

\textbf{Structure-Preserving Architectures like HNNs/LNNs} offer one solution by embedding conservation laws directly into the model's structure \citep{greydanus2019hamiltonian, cranmer2020lagrangian}. By learning a scalar potential (a Hamiltonian or Lagrangian), these models guarantee the conservation of the learned quantity, leading to excellent long-term stability. However, this architectural purity is also their primary weakness: they are ill-suited for the vast majority of real-world engineered systems, which are rarely conservative. They cannot inherently model dissipative effects like friction and struggle to enforce general algebraic constraints that are not derivable from a conservation law \citep{gruver2022deconstructing}. This rigidity limits their applicability and highlights the need for a more flexible mechanism for enforcing arbitrary invariants. 

\textbf{Predict–Project Methods like PNODEs} provide that flexibility \citep{pal2025semi}. By decoupling the dynamics learning from constraint satisfaction, they can enforce any valid algebraic invariant by projecting the predicted state back onto the constraint manifold at each step. While this guarantees physical consistency, existing methods typically pair projection with a \textbf{black-box neural network} that must learn the \textit{entire} dynamics from data. This approach is data-inefficient, forcing the model to waste representational capacity re-learning well-understood physical laws. It also harms interpretability. This specific limitation (the inefficiency of black-box learning) motivates a central tenet of our work: combining the guaranteed enforcement of projection with the data efficiency of a \textit{grey-box, residual learning} approach.

\subsection{The "Semi-Hard" Compromise: Training-Time Guarantees Are Not Enough}
A third category of methods attempts to find a middle ground by enforcing constraints rigorously, but only during the \textit{training loop}. Techniques using Augmented Lagrangian Methods (ALM) or differentiable optimization layers (e.g., OptNet \citep{amos2017optnet}) solve a constrained subproblem at each training step, often leading to better convergence than simple penalties \citep{white2023stabilized}.

However, these methods share a critical vulnerability: constraint satisfaction is a \textit{learned property}, not an \textit{algorithmic guarantee at inference time}. The computationally expensive solver used during training is typically absent during deployment. If the model has not perfectly learned to stay on the manifold, it can still violate constraints at inference time, re-introducing the risk of long-horizon drift. This distinction underscores the importance of our framework's \textit{inference-time} projection, which ensures physical consistency regardless of the learned network's output.

\subsection{Synthesizing a Solution: Our Contributions}
This review shows a clear and unmet need for a \textbf{unified, general-purpose framework} that synthesizes the strengths of these disparate approaches. The core concept of grey-box residual learning has proven effective in application-specific contexts \citep{Nascimento2023-nk, A-46, Arias_Chao2022-rj}, but a formal, generalizable framework is lacking. Our work is designed to fill this gap by integrating three critical properties into a single architecture:
\begin{enumerate}
    \item \textbf{Hard-coded known physics} to avoid the optimization pathologies and lack of guarantees of soft-penalty PINNs.
    \item \textbf{Residual-only learning} to achieve the data efficiency and interpretability that black-box models like PNODEs lack.
    \item \textbf{Principled, inference-time projection} to strictly enforce general algebraic invariants, a guarantee that semi-hard, training-time methods cannot provide.
\end{enumerate}
By formalizing the HRPINN and PHRPINN architectures, we provide a robust, data-efficient, and physically consistent solution for modeling complex dynamical systems.

\section{Background and Problem Formulation}
\label{sec:background}
This section establishes the formal mathematical context for our work. We consider dynamical systems whose state $\mathbf{x}(t) \in \mathbb{R}^n$ evolves according to a set of ODEs that can be separated into a known and an unknown component:

\begin{equation} \label{eq:ode_formulation}
\dot{\mathbf{x}}(t) = \mathbf{f}_{\mathrm{phys}}(\mathbf{x}(t), \mathbf{w}(t); \mathbf{p}) + \mathbf{f}_{\mathrm{unk}}(\mathbf{x}(t), \mathbf{w}(t)), \quad \mathbf{x}(0)=\mathbf{x}_0.
\end{equation}

Here and below we use $n,d,q$ to denote vector dimensions: $\mathbf{x}(t)\in\mathbb{R}^n$ (state of $n$ variables), $\mathbf{w}(t)\in\mathbb{R}^d$ (external inputs of dimension $d$), and $\mathbf{p}\in\mathbb{R}^q$ (time-invariant physical parameters; note that any known, time-varying quantities are treated as external inputs in $\mathbf{w}(t)$). For background on ODE/DAE formulations and index definitions we refer the reader to standard references (e.g., \citep{brenan1995numerical}). 

Here, $\mathbf{f}_{\mathrm{phys}}(\cdot)$ represents the known physics, often derived from first principles, which may depend on external inputs $\mathbf{w}(t) \in \mathbb{R}^d$ and time-invariant physical parameters $\mathbf{p} \in \mathbb{R}^q$.  The term $\mathbf{f}_{\mathrm{unk}}(\cdot)$ represents the unknown or unmodeled residual dynamics, which we approximate with a neural network $\hat{\mathbf{f}}_{\boldsymbol{\theta}}(\cdot)$ parameterized by $\boldsymbol{\theta}$. In many physical systems, the dynamics are also subject to a set of $m$ algebraic invariants that must hold for all time, defining a constraint manifold $\mathcal{M} = \{\mathbf{x} \in \mathbb{R}^n \mid \mathbf{g}(\mathbf{x}) = \mathbf{0}\}$: 

\begin{equation}\label{eq:dae_constraint}
\mathbf{g}(\mathbf{x}(t)) = \mathbf{0}.
\end{equation}

We use this specific form, defining invariants on the differential state alone, as it represents the class of implicit constraints our PHRPINN architecture is designed to enforce via projection. We distinguish this from general explicit constraints (which may depend on inputs $\mathbf{w}$ or separate algebraic variables $\mathbf{z}$) later in this section.

A system described by both \eqref{eq:ode_formulation} and \eqref{eq:dae_constraint} is known as a semi-explicit Differential–Algebraic Equation (DAE) system; see, e.g., \citep{brenan1995numerical} for canonical definitions and index theory.

Examples of index-1 DAEs include electrical circuit models expressed via Modified Nodal Analysis (where the algebraic constraints are Kirchhoff laws) and simple constrained mechanical systems such as a pendulum written in Cartesian coordinates (the fixed-length constraint yields an index-1 DAE after appropriate formulation).

This paper focuses on \textbf{index-1 DAEs}. A DAE is index-1 if the constraint Jacobian, $G(\mathbf{x}) = \partial \mathbf{g} / \partial \mathbf{x}$, has full row rank for all valid states $\mathbf{x} \in \mathcal{M}$. Informally, full row rank means that the $m$ algebraic constraint equations are locally independent (none is a redundant combination of the others), which guarantees a unique local solution for the algebraic correction during projection. This is a critical requirement for the predict--project methods that inspire our work (e.g., PNODEs \citep{pal2025semi}) and for our PHRPINN architecture. The projection step, which enforces the constraint, requires finding a unique correction to the state, a problem whose solvability depends directly on the non-singularity of the matrix system involving $G(\mathbf{x})$. If the DAE were higher-index (index-2 or more), $G(\mathbf{x})$ would be rank-deficient, the projection would be ill-posed, and a unique solution would not be guaranteed \citep{brenan1995numerical}. While many higher-index DAEs in engineering can be converted to index-1 form via symbolic differentiation (a process known as index reduction), this can increase model complexity and amplify noise \citep{brenan1995numerical}. We therefore assume the systems under study are either naturally index-1 or have been pre-processed accordingly.

\paragraph{Explicit vs.\ implicit algebraic constraints.}
We distinguish two practical forms of algebraic constraints. \emph{Explicit} constraint systems take the form $\mathbf{g}(\mathbf{x},\mathbf{z},\mathbf{w})=\mathbf{0}$, where $\mathbf{z}(t)$ denotes algebraic (internal) variables whose values are instantaneously determined by the constraints given the differential states, while $\mathbf{w}(t)$ denotes externally specified inputs (control signals or measured forcings). Such explicit constraints are typically eliminated (solved for $\mathbf{z}$) inside the integrator step. In contrast, \emph{implicit} constraints are invariants on the differential state alone, $\mathbf{g}(\mathbf{x})=\mathbf{0}$, which restrict the trajectory to a manifold $\mathcal{M}$ and typically require a projection step to enforce during time integration.

\subsection{Core Assumptions}
Our theoretical results and algorithmic design rely on the following standard assumptions:
\begin{itemize}
    \item[\textbf{A1}] \textbf{(Problem Class)} The system is an ODE or a semi-explicit index-1 DAE on a compact domain $D\subset\mathbb{R}^n$.
    \newline \textit{Why:} This formally defines the scope of problems our framework is designed to solve.

    \item[\textbf{A2}] \textbf{(Regularity)} The dynamics functions $\mathbf{f}_{\mathrm{phys}}$ and $\mathbf{f}_{\mathrm{unk}}$ are continuously differentiable ($C^1$), and the constraint function $\mathbf{g}$ is twice continuously differentiable ($C^2$).
    \newline \textit{Why:} Smoothness is required to guarantee the existence and uniqueness of solutions and to ensure that the Jacobians needed for backpropagation and the projection step are well-defined.

    \item[\textbf{A3}] \textbf{(Lipschitz Continuity)} The dynamics function $\mathbf{f} = \mathbf{f}_{\mathrm{phys}} + \mathbf{f}_{\mathrm{unk}}$ is Lipschitz continuous on the domain $D$.
    \newline \textit{Why:} This is a fundamental requirement for the well-posedness of the initial value problem and is crucial for proving the stability of the numerical integrators used in our architecture and its theoretical analysis.

    \item[\textbf{A4}] \textbf{(Approximation Family)} The neural network class for $\hat{\mathbf{f}}_{\boldsymbol{\theta}}$ is a universal approximator on compact sets.
    \newline \textit{Why:} This ensures our network is sufficiently expressive to learn any continuous residual dynamics $\mathbf{f}_{\mathrm{unk}}$ to an arbitrary degree of accuracy.

    \item[\textbf{A5}] \textbf{(Integrator Order)} The chosen numerical integrator $\Phi_{\Delta t}$ has a local truncation error of at least order one.
    \newline \textit{Why:} This assumption is necessary to formally bound the global error of our recurrent architecture, as the total error accumulates from the local error at each integration step.

    \item[\textbf{A6}] \textbf{(Constraint Qualification)} The constraint Jacobian $G(\mathbf{x})$ has full row rank along the trajectories of interest. This is known as the Linear Independence Constraint Qualification (LICQ).
    \newline \textit{Why:} This is the most critical assumption for the projection part of our framework. LICQ guarantees that the Karush-Kuhn-Tucker (KKT) system solved during projection has a unique solution, ensuring the projection operator is locally well-defined and differentiable, which is essential for end-to-end training. This is the concrete reason our method is restricted to index-1 systems.

    \item[\textbf{A6'}] \textbf{(Local non-degeneracy / SOSC).} In addition to LICQ (Assumption A6), we assume the following second-order non-degeneracy (second-order sufficient condition, SOSC)\footnote{See, e.g., \textit{Numerical Optimization} by Nocedal \& Wright \citep{nocedal2006numerical} for a standard SOSC statement in equality-constraint NLPs.} holds at the orthogonal projection $x^\star$ of $\tilde x$ onto the constraint manifold $\mathcal{M}=\{x:\; g(x)=0\}$: the Lagrangian Hessian restricted to the tangent space is positive definite at $(x^\star,\lambda^\star)$, i.e.
    \[
    \forall v\in\ker G(x^\star),\qquad v^\top\Big[I + \sum_{i}\lambda_i^\star\nabla^2 g_i(x^\star)\Big] v > 0,
    \]
    where $\ker G(x^\star)$ denotes the kernel (or null space) of the constraint Jacobian (i.e., the tangent space), $\nabla g(x)$ is the gradient (forming the Jacobian $G(x)$), $\nabla^2 g_i$ is the Hessian matrix (matrix of second derivatives) of the $i$-th constraint, and $\lambda^\star$ denotes the Lagrange multiplier vector. Equivalently, the bordered KKT Jacobian in (21) (Appendix \ref{app:Bprojection_derivation}) is nonsingular at $(x^\star,\lambda^\star)$. \newline \textit{Why:} While LICQ (A6) guarantees the KKT Jacobian is invertible locally, this second-order condition strengthens that by ensuring the projection problem itself is well-posed and locally stable, which is crucial for the convergence of solvers used in the projection step and for the theoretical validity of differentiating through it.

\end{itemize}

\paragraph{Scope and Practical Considerations}
It is important to acknowledge that these assumptions, particularly the restriction to index-1 DAEs and the requirement of LICQ, can be restrictive for some real-world applications. Physical systems can exhibit discontinuities, higher-index dynamics, or operate near singular configurations where constraints become degenerate. While a full treatment of such cases is outside our current scope, our framework can be adapted. For instance, if LICQ is violated, the projection step becomes ill-posed. In such scenarios, practical fallbacks include relaxing the hard projection to a soft penalty or employing an Augmented Lagrangian Method (ALM) during training to regularize the problem while still encouraging constraint satisfaction \citep{white2023stabilized}.

\paragraph{Remark.}
For the orthogonal projection objective $ \tfrac12\|x-\tilde x\|^2$, the identity term contributes
a positive definite component and A6$'$ is typically satisfied for sufficiently small perturbations $\tilde x$ around the manifold. We include A6$'$ explicitly to ensure applicability of the implicit
function theorem and to clarify the non-degeneracy required for the projection Jacobian to exist.

\section{Method Overview: HRPINN and PHRPINN}
\label{sec:method_overview}

This section details the architectural and mathematical foundations of our proposed HRPINN and PHRPINN frameworks, connecting their design to the gaps identified in the literature.
To provide a high-level visual comparison of our proposed architectures against the baselines discussed in Section II, Figure~\ref{fig:architectures} illustrates the dataflow of each model.

\begin{figure*}
    \centering
    \includegraphics[width=1.0\linewidth]{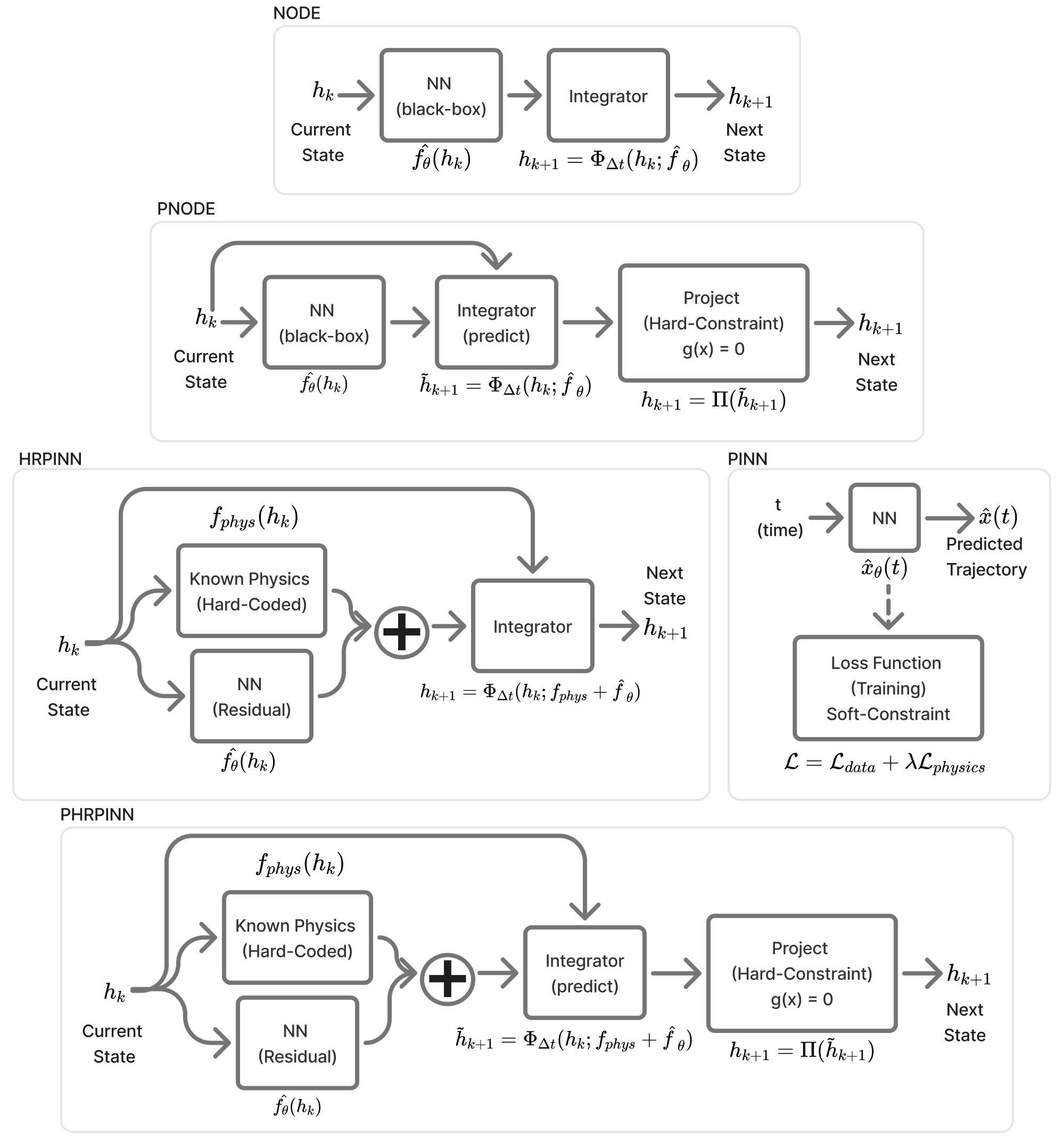}
    \caption{
    A comparative illustration of the PIML architectures discussed.
    \textbf{NODE} learns the entire dynamics $\hat{f}_{\theta}$ as a black-box.
    \textbf{PNODE} augments this with a hard-constraint projection step to enforce invariants $g(x)=0$.
    \textbf{PINN} approximates the solution trajectory $\hat{x}_{\theta}(t)$ directly and enforces physics via a soft penalty in the loss function.
    In contrast, our proposed \textbf{HRPINN} hard-codes the known physics $f_{phys}$ and uses a neural network to learn only the residual dynamics $\hat{f}_{\theta}$.
    Our \textbf{PHRPINN} extends this residual-learning approach by integrating a final projection step, strictly enforcing algebraic invariants $g(x)=0$ by design.}
    \label{fig:architectures}
\end{figure*}

\subsection{The HRPINN Architecture for ODEs}
The HRPINN architecture provides a concrete and purposeful realization of the Universal Differential Equation (UDE) philosophy tailored to time-series learning. HRPINN is a recurrent grey-box model designed for systems governed by ODEs (Eq. \ref{eq:ode_formulation}). Its core is a recurrent cell that hard-codes a numerical integrator for the known physics $\mathbf{f}_{\mathrm{phys}}$, while a neural network learns the residual $\mathbf{f}_{\mathrm{unk}}$. While UDEs specify $\dot{\mathbf{x}} = \mathbf{f}_{\mathrm{phys}} + \hat{\mathbf{f}}_{\boldsymbol{\theta}}$ at a mathematical level, HRPINN makes two concrete architectural choices: (i) structure the model as a recurrent integrator cell; and (ii) train via Backpropagation Through Time (BPTT), which in this context acts as a discrete adjoint sensitivity analysis on the unrolled integrator operations.

The state update rule for the hidden state $\mathbf{h}_k \approx \mathbf{x}(t_k)$ is:
\begin{equation}
\mathbf{h}_{k+1} = \Phi_{\Delta t}\big(\mathbf{h}_k;\, \mathbf{f}_{\mathrm{phys}}(\cdot) + \hat{\mathbf{f}}_{\boldsymbol{\theta}}(\cdot)\big),
\end{equation}
where $\Phi_{\Delta t}$ is a numerical integrator (e.g., Forward Euler, RK4). The model is trained end-to-end on time-series data using Backpropagation Through Time (BPTT) to minimize a sequence loss, e.g., $\mathcal{L}(\boldsymbol{\theta}) = \sum_k \|\mathbf{h}_k - \mathbf{y}_k\|^2$. This design directly enforces the known ODE structure, avoiding the pitfalls of soft-constraint loss balancing \citep{pal2025semi}.

This model is trained end-to-end on time-series data using Backpropagation Through Time (BPTT) to minimize a sequence loss, e.g.,
\[
\mathcal{L}(\boldsymbol{\theta}) = \sum_{k=0}^{K} \ell_{\mathrm{data}}(\mathbf{h}_k,\mathbf{y}_k),
\]
where $\{\mathbf{y}_k\}$ are the observed time-series measurements corresponding to the true states (or observables) at times $t_k$, and $\ell_{\mathrm{data}}$ is typically the Mean Squared Error (MSE). The overall training process is summarized in Algorithm~\ref{alg:hrpinn}.

\begin{algorithm}[H]
\caption{HRPINN Training Step}
\label{alg:hrpinn}
\begin{algorithmic}[1]
\State \textbf{Input:} Initial state $\mathbf{h}_0$, observations $\{\mathbf{y}_k\}$, physics $\mathbf{f}_{\mathrm{phys}}$, parameters $\boldsymbol{\theta}$, integrator $\Phi_{\Delta t}$.
\State \textbf{Forward Pass:}
\For{$k=0$ to $K-1$}
    \State $\mathbf{h}_{k+1} \leftarrow \Phi_{\Delta t}\big(\mathbf{h}_k;\, \mathbf{f}_{\mathrm{phys}}(\mathbf{h}_k) + \hat{\mathbf{f}}_{\boldsymbol{\theta}}(\mathbf{h}_k)\big)$ \Comment{Unroll trajectory}
\EndFor
\State Compute loss $\mathcal{L}(\boldsymbol{\theta})$ based on $\{\mathbf{h}_k\}$ and $\{\mathbf{y}_k\}$.
\State \textbf{Backward Pass:}
\State Compute gradient $\nabla_{\boldsymbol{\theta}}\mathcal{L}$ via BPTT.
\State Update parameters: $\boldsymbol{\theta} \leftarrow \boldsymbol{\theta} - \eta \nabla_{\boldsymbol{\theta}}\mathcal{L}$.
\end{algorithmic}
\end{algorithm}

\subsection{The PHRPINN Extension for DAEs}
For systems with algebraic constraints (DAEs), we extend HRPINN to PHRPINN (Projected HRPINN). This architecture combines the recurrent residual learning of HRPINN with a manifold projection step to enforce the algebraic constraint $\mathbf{g}(\mathbf{x})=\mathbf{0}$. This approach is directly inspired by classical numerical solvers for semi-explicit DAEs and the PNODE methodology \citep{pal2025semi}.

The PHRPINN cell performs a two-stage update at each time step:
\begin{enumerate}
    \item \textbf{Predict:} An intermediate, unconstrained state $\tilde{\mathbf{h}}_{k+1}$ is computed using the standard HRPINN update rule. This step integrates the dynamics forward in time, but the resulting state may drift off the constraint manifold $\mathcal{M} = \{\mathbf{x} \mid \mathbf{g}(\mathbf{x}) = \mathbf{0}\}$.
    \begin{equation}
        \tilde{\mathbf{h}}_{k+1} = \Phi_{\Delta t}\big(\mathbf{h}_k;\, \mathbf{f}_{\mathrm{phys}} + \hat{\mathbf{f}}_{\boldsymbol{\theta}}\big).
    \end{equation}
    \item \textbf{Project:} The intermediate state $\tilde{\mathbf{h}}_{k+1}$ is corrected by projecting it back onto the manifold $\mathcal{M}$ to find the final state $\mathbf{h}_{k+1}$. This projection finds the closest point on the manifold to the predicted point.
\end{enumerate}

\begin{definition}[Orthogonal Manifold Projection]
The projection operator $\Pi$ maps a point $\tilde{\mathbf{x}}$ to the closest point $\mathbf{x}^*$ on the constraint manifold $\mathcal{M}$ by solving the constrained optimization problem:
\begin{equation}
\mathbf{x}^* = \Pi(\tilde{\mathbf{x}}) = \arg\min_{\mathbf{x}} \frac{1}{2}\|\mathbf{x} - \tilde{\mathbf{x}}\|_2^2 \quad \text{s.t.} \quad \mathbf{g}(\mathbf{x}) = \mathbf{0}.
\end{equation}
Using Lagrange multipliers $\boldsymbol{\lambda}$, the solution is characterized by the Karush-Kuhn-Tucker (KKT) conditions, which form a system of nonlinear equations:
\begin{subequations}
\label{eq:phrpinn-kkt}
\begin{align}
\mathbf{x}^* - \tilde{\mathbf{x}} + G(\mathbf{x}^*)^\top \boldsymbol{\lambda} &= \mathbf{0} \quad &\text{(Stationarity)} \\
\mathbf{g}(\mathbf{x}^*) &= \mathbf{0} \quad &\text{(Primal Feasibility)}
\end{align}
\end{subequations}
where $G(\mathbf{x}) = \partial \mathbf{g}/\partial \mathbf{x}$ is the constraint Jacobian.
\end{definition}

This is the standard formulation for finding the closest point on a manifold via orthogonal projection \citep{nocedal2006numerical}.

This two-stage process yields a hybrid architecture that is physics-informed in two ways: (i) known ODE dynamics are embedded within the integrator, and (ii) algebraic invariants are strictly enforced by the projection step.

For the projection step, PHRPINN supports two strategies, analogous to PNODEs \citep{pal2025semi}: a \emph{robust} variant that solves the full nonlinear KKT system in \eqref{eq:phrpinn-kkt} to a given tolerance (e.g., using a Newton method), and a \emph{fast} single-factorization variant that linearizes the system around the predicted state $\tilde{\mathbf{h}}_{k+1}$ to reuse a single matrix factorization per step.

\subsection{Differentiability of the Projection Operator}

We compute gradients using discrete adjoints for the time integrator and implicit differentiation of the projection KKT system. Continuous-adjoint strategies may require constructing a time-continuous interpolant of the discrete trajectory; such interpolants can leave the constraint manifold and create numerical instability in the adjoint computations. For these reasons (and in line with recent practical studies such as PNODE \citep{pal2025semi}) we adopt discrete adjoints and implicit differentiation, which yield stable, well-defined gradients for the predict–project architecture.

To train the PHRPINN model end-to-end using BPTT, we must compute gradients through the projection step $\mathbf{x}^* = \Pi(\tilde{\mathbf{x}})$. Since the projection is defined implicitly by the KKT system \eqref{eq:phrpinn-kkt}, we use the implicit function theorem to find its Jacobian, following the approach used in PNODE and OptNet \citep{pal2025semi, amos2017optnet}. This process involves solving the adjoint of a large linear system known as the KKT sensitivity system. For a full derivation of this system and a detailed analysis of the projector's Jacobian, please see Appendix \ref{app:Bprojection_derivation}.

\paragraph{Tangent-Space Closed-Form (Hessian Neglected).}
The \emph{robust} projection, which solves the full KKT system \eqref{eq:phrpinn-kkt}, can be computationally expensive. We therefore adopt the \emph{fast} tangent-space projector, which is also used in the PNODE framework \citep{pal2025semi}. This method uses an efficient, explicit formula for the projection Jacobian:
\begin{equation}
J_\Pi \;=\; I - G(\mathbf{x}^*)^\top\,(G(\mathbf{x}^*) G(\mathbf{x}^*)^\top)^{-1} \, G(\mathbf{x}^*).
\label{eq:tangent_projector}
\end{equation}
This is the standard formula for an orthogonal projector onto the tangent space of the manifold at $\mathbf{x}^*$. It is computationally efficient when the number of constraints $m$ is small, as it only requires solving an $m \times m$ linear system. The PNODE authors justify this formula as the result of a single Newton step used to solve the KKT conditions \citep{pal2025semi}. In our Appendix \ref{app:Bprojection_derivation}, we provide an alternative justification via formal sensitivity analysis, showing this formula is a principled approximation derived by neglecting the Hessian (i.e., manifold-curvature) terms. For numerical stability, we compute the projector by first finding an orthonormal basis $Q$ for the row space of $G$ via a QR factorization (from \textbf{Q}rthogonal and upper \textbf{R}iangular) or SVD, and then forming $J_\Pi = I - Q^\top Q$. This avoids explicitly forming $G G^\top$, which can be severely ill-conditioned.

\paragraph{Scope and Edge Cases.}
Our differentiability analysis covers equality-constrained, index-1 DAEs where the Linear Independence Constraint Qualification (LICQ, A6) holds. If LICQ fails (e.g., constraints become redundant), if constraints include inequalities, or if the manifold self-intersects, the projection may become non-unique or non-differentiable (as the mapping from $\tilde{\mathbf{x}}$ to $\mathbf{x}^*$ is no longer smooth). This class includes many well-behaved constrained mechanical systems (like multi-body pendulums away from singular configurations) or power grid models. In our work, we assume that system trajectories remain in regions where LICQ holds and the KKT matrix is non-singular. As a practical fallback for systems that might violate these conditions, we recommend using an augmented Lagrangian (ALM) or a smooth penalty relaxation during training to maintain differentiability \citep{white2023stabilized}.
\begin{algorithm}[H]
\caption{PHRPINN Training Step}
\label{alg:phrpinn}
\begin{algorithmic}[1]
\State \textbf{Input:} Initial state $\mathbf{h}_0$, observations $\{\mathbf{y}_k\}$, physics $\mathbf{f}_{\mathrm{phys}}$, constraints $\mathbf{g}$, parameters $\boldsymbol{\theta}$.
\State \textbf{Forward Pass:}
\For{$k=0$ to $K-1$}
    \State $\tilde{\mathbf{h}}_{k+1} \leftarrow \Phi_{\Delta t}\big(\mathbf{h}_k;\, \mathbf{f}_{\mathrm{phys}} + \hat{\mathbf{f}}_{\boldsymbol{\theta}}\big)$ \Comment{Predict step}
    \State $\mathbf{h}_{k+1} \leftarrow \text{Project}(\tilde{\mathbf{h}}_{k+1}, \mathbf{g})$ \Comment{Solve KKT system \eqref{eq:phrpinn-kkt}}
\EndFor
\State Compute loss $\mathcal{L}(\boldsymbol{\theta})$.
\State \textbf{Backward Pass:}
\State Compute gradient $\nabla_{\boldsymbol{\theta}}\mathcal{L}$ via BPTT, using implicit differentiation for the projection step.
\State Update parameters: $\boldsymbol{\theta} \leftarrow \boldsymbol{\theta} - \eta \nabla_{\boldsymbol{\theta}}\mathcal{L}$.
\end{algorithmic}
\end{algorithm}
\paragraph{Local uniqueness of orthogonal projection.}
Orthogonal projection onto a smooth manifold is locally unique only within the manifold's normal injectivity radius. Accordingly, differentiability and Newton convergence for the projection are
local properties: the predictor $\tilde x$ must remain sufficiently close to $\mathcal M$ for uniqueness to hold. Practically, this motivates choosing predictor steps that keep $\tilde x$ within a tubular neighborhood of $\mathcal M$; if the predictor leaves this neighborhood, projection may become non-unique and additional regularization or continuation strategies are required. The full predict–project training loop for PHRPINN, which combines the HRPINN integration with an orthogonal projection at each step, is detailed in Algorithm~\ref{alg:phrpinn}.

\subsection{Theoretical Result}
The architectural choice of HRPINN is justified by its ability to represent the same solution space as standard PINNs, while offering practical advantages. We formalize this equivalence with the following theorem, which states that under standard regularity and approximation assumptions, a trajectory generated by an HRPINN can be approximated by a standard PINN, and vice versa.

\paragraph{Clarification (scope of Theorem~\ref{thm:equivalence}).}
Theorem~\ref{thm:equivalence} establishes \emph{representational equivalence} between HRPINN and PINN formulations: each framework can in principle represent the same set of trajectories under the stated regularity and approximation assumptions. It does \textbf{not} make claims about the practical difficulty of optimization, convergence speed, or sample complexity. The result assumes exact initial
states (error $e_0=0$) and requires the universal approximator class to be capable of approximating
the relevant functions on a compact set that contains the true trajectory and a tubular neighborhood
around it (to account for perturbed states during integration). Our conjectures that follow are meant as practical motivations (hypotheses) about optimization conditioning and generalization; they are evaluated empirically in Section~V and revisited in the Discussion to assess how the data supports them.

\begin{theorem}[Representational Equivalence]
\label{thm:equivalence}
Suppose Assumptions A1--A6, and A6$'$ hold, the initial state is known exactly ($\mathbf{e}_0 = \mathbf{0}$), and the UAT holds on a compact tubular neighborhood containing the true system trajectory. Let the true system solution be $\mathbf{x}^*(t)$ and the unknown physics be $\mathbf{f}_{\mathrm{unk}}$.

\begin{enumerate}
    \item \textbf{(PINN $\Rightarrow$ HRPINN)} If a standard PINN can approximate $\mathbf{x}^*(t)$ and $\mathbf{f}_{\mathrm{unk}}$, then there exists an HRPINN whose learned component $\hat{\mathbf{f}}_{\boldsymbol{\phi}}$ can approximate $\mathbf{f}_{\mathrm{unk}}$ such that the HRPINN's trajectory $\mathbf{h}_k$ remains arbitrarily close to the true solution samples $\mathbf{x}^*(t_k)$, with an error bound dependent on the integrator's Local Truncation Error (LTE) and the network approximation error.

    \item \textbf{(HRPINN $\Rightarrow$ PINN)} Conversely, if an HRPINN's trajectory $\mathbf{h}_k$ successfully approximates the true solution samples $\mathbf{x}^*(t_k)$, then there exist neural networks for a standard PINN that can approximate the continuous trajectory $\mathbf{x}^*(t)$ and the unknown physics $\mathbf{f}_{\mathrm{unk}}$, such that the PINN's physics-based residual can be made arbitrarily small.
\end{enumerate}
\end{theorem} 

\begin{proof}[Proof Sketch]
The full proof relies on two core principles: the Universal Approximation Theorem (UAT) and standard error bounds for numerical integrators. The UAT guarantees that neural networks in both frameworks have the capacity to approximate the necessary continuous functions (either the solution $\mathbf{x}^*(t)$ or the unknown dynamics $\mathbf{f}_{\mathrm{unk}}$). The proof then uses a discrete Grönwall inequality to show that the HRPINN's trajectory error is bounded by the sum of its neural network's approximation error and the integrator's LTE. Conversely, we show that by constructing a specialized $C^1$ cubic Hermite interpolant through the HRPINN's discrete states, we can guarantee bounded error on both the solution \textit{and its derivatives}. This allows a standard PINN, via the UAT, to approximate both the continuous trajectory and its corresponding physics residual. A complete, formal proof with all preliminaries, assumptions, and derivations is provided in Appendix \ref{app:Ctheoretical_foundations}.
\end{proof}

We do not claim a formal proof of superior optimization or generalization, but we provide empirical evidence for these conjectures in our experiments.

\subsection{Conjectured Advantages of the Hard-Constraint Approach}
The motivation for the HRPINN's design lies in two primary conjectured advantages over soft-constrained PINNs.

\begin{conjecture}[Improved Optimization Conditioning]
The HRPINN architecture may foster a more stable and amenable optimization landscape. By hard-coding known physics, the neural network $\hat{\mathbf{f}}_\phi$ learns a potentially simpler residual function. This avoids forcing a single network $\hat{\mathbf{x}}_\theta$ to satisfy stiff differential constraints, which can create ill-conditioned loss landscapes and gradient pathologies in standard PINNs \citep{wang2021understanding, urban2025unveiling}. Mechanistically, the recurrent structure's state transition Jacobian can be kept well-conditioned by selecting an appropriate integrator time step $\Delta t$, which directly controls the eigenvalues of the Jacobian and thus mitigates the vanishing and exploding gradient problems in BPTT.
\end{conjecture}

\begin{conjecture}[Superior Generalization]
By focusing the learning task on only the unknown dynamics $\mathbf{f}_{\text{unk}}$, HRPINNs may generalize better than standard PINNs from finite data. If the unknown residual function $\mathbf{f}_{\mathrm{unk}}(\mathbf{x},t)$ (the learning target for HRPINN) has a lower intrinsic complexity than the full state trajectory $\mathbf{x}(t)$ (the learning target for a standard PINN), the HRPINN can achieve an accurate approximation with a neural network of lower effective complexity. As supported by statistical learning theory, models with lower complexity (e.g., smaller Rademacher complexity) tend to have tighter generalization bounds and exhibit better performance on unseen data \citep{shalev2014understanding}. Our experimental results provide empirical evidence for this conjecture. The strength of this benefit is highly system-dependent, proving most effective when: (i) the known physics $\mathbf{f}_{\mathrm{phys}}$ represent a significant portion of the dynamics, (ii) the unknown residual $\mathbf{f}_{\mathrm{unk}}$ is a relatively simple function for the network to learn, and (iii) the available training data is sparse, making the structural prior particularly valuable. We demonstrate that it can hold strongly under these conditions, though its benefits are not universal.
\end{conjecture} 

\subsection{A 7-Step Method for Building HRPINNs}
We present a practical, seven-step process for constructing HRPINNs for a wide range of CPS. For concreteness, we use a spring–mass–damper illustration.

\textbf{Step 1: Identify the system of equations.} Start from ODEs or semi-explicit index-1 DAEs with outputs:
\begin{equation}
\begin{cases}
\dot{\mathbf{x}}(t) = \mathbf{f}(\mathbf{x}(t),\mathbf{z}(t),\mathbf{w}(t),\mathbf{p}, t, \mathbf{f}_{\mathrm{unk}}) \\
\mathbf{0} = \mathbf{h}(\mathbf{x}(t),\mathbf{z}(t),\mathbf{w}(t),\mathbf{p}, t, \mathbf{h}_{\mathrm{unk}}) \\
\mathbf{y}(t) = \mathbf{g}(\mathbf{x}(t),\mathbf{z}(t),\mathbf{w}(t),\mathbf{p}, t, \mathbf{g}_{\mathrm{unk}})
\end{cases}
\end{equation}
\textit{Example:} $m\ddot{x}(t) + f_{\mathrm{damp,unk}}(x, \dot{x}) + kx(t) = F_{ext}(t)$ with state $\mathbf{x}_{sm}=[x,\dot{x}]^\top$ and
{\footnotesize
\begin{equation}
\dot{\mathbf{x}}_{sm}(t)=\begin{bmatrix} x_2(t) \\ \tfrac{1}{m}\big(w(t)-k x_1(t)- f_{\mathrm{damp,unk}}(x_1(t),x_2(t))\big) \end{bmatrix}.
\end{equation}}

\textbf{Step 2: Categorize variables and parameters.} Separate known time-invariant parameters $\mathbf{p}_k$, unknown time-invariant parameters $\mathbf{p}_u$ (to be estimated), known inputs $\mathbf{w}(t)$, outputs $\mathbf{y}(t)$, states $\mathbf{x}(t)$, and algebraic variables $\mathbf{z}(t)$.

\textbf{Step 3: Define HRPINN inputs and outputs.} At discrete times $t_k$, typically set inputs $\mathbf{W}_{\mathrm{HRPINN}}(t_k)=\mathbf{w}(t_k)$ and outputs $\mathbf{Y}_{\mathrm{HRPINN}}(t_k)=\mathbf{y}(t_k)$.

\textbf{Step 4: Manipulate and discretize equations.} Isolate known physics from unknown components and discretize over time. Using Forward Euler:
\begin{equation}
\mathbf{x}_{k+1} = \mathbf{x}_k + \Delta t\,\mathbf{f}_{\mathrm{eff}}(\mathbf{x}_k, \mathbf{z}_k, \mathbf{w}_k, \mathbf{p}, t_k; \boldsymbol{\phi}).
\end{equation}
\textit{Example:}
\begin{align}
x_{1,k+1} &= x_{1,k} + \Delta t\, x_{2,k},\\
x_{2,k+1} &= x_{2,k} + \Delta t\,\Big( \underbrace{\tfrac{1}{m}(w_k - k x_{1,k})}_{\text{known}} - \underbrace{\tfrac{1}{m}\,\hat f_{\mathrm{damp}}(x_{1,k},x_{2,k};\boldsymbol{\phi})}_{\text{unknown}} \Big).
\end{align}

\textbf{Step 5: Design the HRPINN architecture.} Use a recurrent cell with state $\mathbf{h}_k\approx\mathbf{x}_k$ and numerical update
\begin{equation}
\mathbf{h}_{k+1}=\mathbf{h}_k + \Delta t\,\big( \mathbf{f}_{\mathrm{phys}}(\mathbf{h}_k,\mathbf{w}_k,\mathbf{p}_{\mathrm{known}},t_k) + \hat{\mathbf{f}}_{\boldsymbol{\phi}}(\mathbf{h}_k,\mathbf{w}_k,t_k) \big).
\end{equation}

\textbf{Step 5b (for PHRPINN): Add a Projection Layer.} If the system is governed by implicit algebraic invariants $\mathbf{g}(\mathbf{x})=\mathbf{0}$, augment the recurrent cell with a projection layer. After each integration step, this layer corrects the predicted state by solving the KKT system to find the closest point on the constraint manifold.

\textbf{Step 6: Calibrate the model.} Train by minimizing sequence loss against observed data using BPTT; optionally add penalties (e.g., parameter positivity) and regularization. A typical objective is $\mathcal{L}=\sum_k \ell_{\mathrm{data}}(\mathbf{h}_k,\mathbf{y}_k)+\lambda\,\mathcal{R}(\boldsymbol{\phi})$, where $\mathcal{R}(\boldsymbol{\phi})$ is a regularization term (e.g., L2 weight decay on network parameters $\boldsymbol{\phi}$) and $\lambda$ is a regularization hyperparameter.

\textbf{Step 7: Use the model for inference.} Deploy the trained HRPINN for prediction under new inputs. Longer-timescale parameter evolution is outside this scope, but using a separate model to update the parameters of the system reflecting complex degradation patterns is an interesting future work.

\subsection{Complexity and Gradient Strategy}

We briefly analyze the computational complexity and memory requirements of the gradient computation strategies:

\begin{itemize}
  \item \textbf{Unrolled BPTT:} Memory $\mathcal{O}(T\cdot n)$ to store states/activations for the backward pass; backward cost is a constant-factor multiple of the forward cost. This is the standard approach for HRPINN.
  \item \textbf{Implicit Projection Differentiation:} Memory $\mathcal{O}(n)$ (no unrolling of the projection solve) but requires solving a symmetric indefinite linear system per step in the backward pass. Dense factorization costs $\mathcal{O}((n+m)^3)$, while iterative or sparse solvers can reduce this to an effective cost of $\mathcal{O}((n+m)\kappa)$, where $\kappa$ is the condition number.
  \item \textbf{Practice:} For PHRPINN, use unrolled BPTT for the integrator part and implicit gradients for the projection part. For systems with long horizons or computationally expensive projections, checkpointing or other memory-saving BPTT variants can be combined with implicit gradients for the projection. Exploiting sparsity in the constraint Jacobian $G$ is crucial for performance on large-scale problems.
\end{itemize}

\subsection{Practical Considerations and Architectural Choice}
The choice between the HRPINN and PHRPINN architectures involves critical trade-offs between computational cost, implementation complexity, and the nature of the system's constraints.

\paragraph{Trade-offs and Costs}
While HRPINN provides a lightweight and efficient way to enforce known ODEs (exactly because it avoids the computationally expensive projection), it offers no mechanism for implicit algebraic invariants. PHRPINN closes this gap by guaranteeing constraint satisfaction, but at a cost. The projection step introduces significant computational overhead, requiring the solution of a KKT system (or its linear approximation) at every time step, both in the forward pass and during backpropagation via implicit differentiation. This can increase training and inference time substantially compared to a standard HRPINN.

\paragraph{Dependency on Model Fidelity}
Furthermore, the effectiveness of both architectures relies on the quality of the known physical model, $\mathbf{f}_{\mathrm{phys}}$. By hard-coding these physics, the framework assumes they are a reasonably accurate representation. If the base model $\mathbf{f}_{\mathrm{phys}}$ is significantly misspecified, this can hinder the learning of the true residual dynamics. This issue is potentially \textbf{more acute for PHRPINN}, as it not only integrates potentially incorrect dynamics but also projects onto a constraint manifold $\mathcal{M}=\{\mathbf{x} \mid \mathbf{g}(\mathbf{x})=\mathbf{0}\}$ that may itself be derived from the flawed physics. Rigidly enforcing consistency with an incorrect manifold can be more detrimental than simply integrating slightly incorrect dynamics, potentially leading to a poorer overall model fit.

\paragraph{Guidance for Practitioners}
Based on these considerations, we propose the following guidance:
\begin{itemize}
    \item \textbf{Use HRPINN when:} The system is governed by ODEs or DAEs with only \textbf{explicit} algebraic constraints that can be solved within the integrator step. It is the ideal choice when the primary goal is residual learning without needing to enforce implicit, long-horizon invariants.
    \item \textbf{Use PHRPINN when:} The system has \textbf{implicit} algebraic invariants (e.g., conservation of energy, momentum) that must be strictly satisfied to prevent non-physical long-term drift. It is necessary when physical consistency is paramount, and the additional computational cost is acceptable for the application.
\end{itemize}

\section{Case Studies}
\label{sec:casestudies}

To validate our proposed frameworks, we designed experiments to address five key validation objectives (VOs) that probe their core capabilities, theoretical conjectures, robustness, and practical utility:

\begin{itemize}
    \item \textbf{VO1:} To evaluate the effectiveness of the core HRPINN architecture at handling explicit invariants in a real-world DAE system.
    \item \textbf{VO2:} To assess the data efficiency of the HRPINN architecture and its ability to produce high-fidelity nominal models for complex, partially observed systems.
    \item \textbf{VO3:} To determine if the PHRPINN extension strictly enforces implicit invariants while maintaining superior accuracy over baselines, and to analyze the associated trade-offs.
    \item \textbf{VO4:} To investigate the robustness of the PHRPINN framework, its effect on the optimization landscape, and its generalization from finite data.
    \item \textbf{VO5:} To verify whether hard-constrained projection is necessary for ensuring physical consistency in systems with implicit invariants, compared to soft-constrained alternatives.
\end{itemize}

We address these objectives through two complementary case studies. \textbf{Case Study 1} uses a real-world Lithium-ion battery prognostic task to validate the core HRPINN architecture in a partially observed DAE setting (addressing VO1 \& VO2). \textbf{Case Study 2} uses a suite of standard constrained dynamical systems to rigorously evaluate the PHRPINN extension, focusing on its ability to enforce implicit algebraic invariants and ensure long-term stability (addressing VO3, VO4, \& VO5).

\subsection{Case Study 1: Lithium-Ion Battery Prognostics}
\label{sec:paper_case_battery}

This first case study, inspired by previous work on the topic \citep{Nascimento2023-nk}, focuses on a Lithium-ion battery, a common component in many Cyber-Physical Systems where predicting performance degradation is critical.

\subsubsection{System, Data, and Preprocessing}
We use the NASA PCoE Li-ion battery collection \citep{Dataset11}, a standard public dataset for prognostics research. To focus on modeling the nominal (healthy) behavior of the cells, we utilize the first two discharge cycles from all available batteries. These initial cycles are ideal as they exhibit minimal effects from the complex degradation mechanisms that appear later in a battery's life. The task is to predict the terminal voltage $V(t)$ over a full discharge cycle given the applied current profile $i_{\text{app}}(t)$ and the cell temperature $T(t)$.

Each full discharge profile is treated as an independent time-series for training and evaluation. For data preprocessing, cycles were identified according to the protocol from the data provider. To ensure data quality, only complete discharge cycles were used, and any cycles with clear data acquisition errors or anomalies were discarded. Anomalies were identified as cycles containing non-monotonic voltage drops during discharge, significant sensor noise, or incomplete data logs, as determined by automated scripts and visual inspection.

The raw data is used directly without normalization. This choice is deliberate to preserve the physical scale of the variables, which is important for the physics-based components of the models. 
We acknowledge that this lack of normalization presents a significant challenge, particularly for the purely data-driven NODE and PINN baselines, which are known to benefit from scaled inputs. However, we maintain this approach for methodological consistency and to ensure a direct and unambiguous comparison. By providing all models with the exact same raw, unscaled data, we are explicitly testing the inherent robustness of each \textit{architecture} to the varying magnitudes of real-world physical quantities, rather than the effectiveness of a model-specific preprocessing pipeline.

\subsubsection{Governing Equations}
The physical model for the battery is a semi-explicit index-1 Differential-Algebraic Equation (DAE) system, adapted from the literature \citep{nasabatterymodel}. This structure is common in electrochemical modeling, where some states evolve dynamically according to differential laws (e.g., charge accumulation) while others are linked through instantaneous algebraic relationships (e.g., potentials and reaction kinetics). 

The system can be formally expressed as:

\begin{align}
\dot{\mathbf{x}}(t) &= f(\mathbf{x}(t), \mathbf{z}(t), w(t)) \\
\mathbf{0} &= g(\mathbf{x}(t), \mathbf{z}(t), w(t))
\end{align}

where $\mathbf{x}(t)$ is the vector of differential state variables, $\mathbf{z}(t)$ is the vector of algebraic variables, and $w(t) = i_{app}(t)$ is the external input current. For this specific model, the differential state vector $\mathbf{x}(t)$ is defined by the states in Eq. \ref{eq:paper_battery_differential} (e.g., $q_{b,p}, V'_o$), while the algebraic vector $\mathbf{z}(t)$ consists of all instantaneous variables defined by the constraints in Eq. \ref{eq:paper_battery_algebraic} (e.g., $V_{U,i}, J_i, C_{b,i}$).
The system is designated as index-1 because the algebraic constraints can be directly solved for the algebraic variables $\mathbf{z}$ at any time $t$, given the values of the differential variables $\mathbf{x}$ at that same time, without needing to differentiate the constraint equations. Here, the function $\mathbf{g}(\cdot)$ refers to the complete set of algebraic constraints presented in Eq. \ref{eq:paper_battery_algebraic}. This is formally guaranteed because the Jacobian of $\mathbf{g}$ with respect to the algebraic variables $\mathbf{z}$ ($\partial \mathbf{g} / \partial \mathbf{z}$) is non-singular, a standard condition for index-1 DAEs that ensures a unique solution for $\mathbf{z}$ exists given $\mathbf{x}$ \citep{brenan1995numerical}.

The differential state vector is defined as $\mathbf{x} = [q_{b,p}, q_{s,p}, q_{b,n}, q_{s,n}, V'_o, V'_{\eta,p}, V'_{\eta,n}]^T$. These states represent the charge in the bulk ($q_b$) and surface ($q_s$) regions of the positive (p) and negative (n) electrodes, along with filtered versions of the ohmic voltage drop ($V'_o$) and electrode overpotentials ($V'_{\eta}$). The governing differential equations, which define the function $f(\cdot)$, are:
\begin{equation}
\label{eq:paper_battery_differential}
\begin{aligned}
    \dot{q}_{s,p} &= i_{app} + \dot{q}_{b,s,p}, & \dot{q}_{b,p} &= -\dot{q}_{b,s,p} \\
    \dot{q}_{s,n} &= -i_{app} + \dot{q}_{b,s,n}, & \dot{q}_{b,n} &= -\dot{q}_{b,s,n} \\
    \dot{V}_o' &= (V_o - V_o')/\tau_o, & \dot{V}_{\eta,p}' &= (V_{\eta,p} - V_{\eta,p}')/\tau_{\eta,p}, \\ \dot{V}_{\eta,n}' &= (V_{\eta,n} - V_{\eta,n}')/\tau_{\eta,n}
\end{aligned}
\end{equation}

The algebraic variables $\mathbf{z}$ are determined by the set of algebraic constraints $g(\cdot)=0$.
A key component of this model is the Redlich-Kister (RK) expansion for the term $V_{INT,i}$, which accounts for the non-ideal thermodynamic behavior of the electrodes. This term is particularly important as it represents the complex physical component that we treat as unknown and replace with a neural network in our HRPINN implementation. The full set of algebraic constraints is:
\begin{equation}
\label{eq:paper_battery_algebraic}
\begin{aligned}
    &V - (V_{U,p} - V_{U,n} - V_o' - V'_{\eta,p} - V'_{\eta,n}) = 0 \\
    &V_{U,i} - \left( U_{0,i} + \frac{RT}{nF} \ln \left( \frac{1-x_i}{x_i} \right) + V_{INT,i} \right) = 0, \quad i \in \{p,n\} \\
    &V_{INT,i} - \frac{1}{nF} \sum_{k=0}^{N_i} A_{i,k} \left( (2x_i - 1)^{k+1} \right. \\
    & \qquad \left. - \frac{k(2x_i-1) - (x_i-1)}{(2x_i - 1)^{1-k}} \right) = 0, \quad i \in \{p,n\} \\
    &V_{\eta,i} - \frac{2RT}{F} \operatorname{arcsinh} \left( \frac{J_i}{2J_{i0}} \right) = 0, \quad i \in \{p,n\} \\
    &J_i - \frac{i_{app}}{S_i} = 0, \quad i \in \{p,n\} \\
    &J_{i0} - k_i (1-x_{s,i})^\alpha (x_{s,i})^{1-\alpha} = 0, \quad i \in \{p,n\} \\
    &V_o - i_{app}R_o = 0 \\
    &\dot{q}_{b,s,i} - \frac{1}{D}(C_{b,i} - C_{s,i}) = 0, \quad i \in \{p,n\} \\
    &C_{b,i} - \frac{q_{b,i}}{v_{b,i}} = 0, \quad i \in \{p,n\} \\
    &C_{s,i} - \frac{q_{s,i}}{v_{s,i}} = 0, \quad i \in \{p,n\} \\
    &x_i - \frac{q_{b,i} + q_{s,i}}{q_{max}} = 0, \quad i \in \{p,n\} \\
    &x_{s,i} - \frac{q_{s,i}}{q_{max,s,i}} = 0, \quad i \in \{p,n\}
\end{aligned}
\end{equation}
In the above, the index $i$ denotes the electrode, either positive ($p$) or negative ($n$), and the final constraint assumes a surface-based definition for the state of charge. A comprehensive list of all symbols and their definitions is provided in Table \ref{tab:paper_battery_notation_defs}. The specific parameter values used in this study, listed in Table \ref{tab:paper_battery_params_values}, are taken from the NASA Prognostics Center of Excellence data set \citep{Dataset11}.

\subsubsection{Model Implementations for the Battery Case Study}
To evaluate the proposed HRPINN architecture, we compare it against several established baselines. For this specific DAE system, the models are defined as follows:

\paragraph{HRPINN (Proposed)} The HRPINN model is implemented as a recurrent cell whose state corresponds to the differential state vector $\mathbf{x}$. The known physics are hard-coded, while the unknown or complex components are replaced by neural networks. Specifically, the computationally expensive Redlich-Kister expansion for the non-ideal thermodynamic voltage, $V_{INT,i}$, is replaced by two separate MLPs, $\hat{V}_{INT,p}(\mathbf{x};\theta_p)$ and $\hat{V}_{INT,n}(\mathbf{x};\theta_n)$.
A carefully designed MLP architecture is not only capable of accurately approximating this complex term but can also be significantly more computationally efficient at inference time than evaluating the full, high-order polynomial expansion. The cell's update rule is a numerical DAE solver (e.g., Forward Euler) that enforces the known physics by design. This architecture hard-codes the trusted physics and focuses the learning task on the complex, non-ideal residual terms.

\paragraph{Neural ODE (NODE)} As a black-box baseline, the NODE learns the entire system dynamics from data. It uses the same recurrent integrator structure as HRPINN, but the right-hand side of the differential equation is replaced entirely by a neural network $\hat{f}_{\theta}(\mathbf{x}, w)$. It has no explicit knowledge of the underlying electrochemical laws and must learn them implicitly from the data.

\paragraph{Physics-Informed Neural Network (PINN)} The PINN baseline uses a single neural network $\hat{\mathbf{x}}_{\theta}(t)$ to directly approximate the trajectory of the differential states. It is trained by minimizing a composite loss function that penalizes deviations from both the observed data and the governing equations:
\begin{equation}
    \mathcal{L}_{PINN} = \mathcal{L}_{data} + \lambda_{diff} \mathcal{L}_{diff} + \lambda_{alg} \mathcal{L}_{alg}
\end{equation}
where $\mathcal{L}_{data}$ measures the mismatch with observations, $\mathcal{L}_{diff}$ penalizes the residual of the differential equations (Eq. \ref{eq:paper_battery_differential}), and $\mathcal{L}_{alg}$ penalizes the residual of the algebraic constraints (Eq. \ref{eq:paper_battery_algebraic}). The derivatives required for $\mathcal{L}_{diff}$ are computed via automatic differentiation.

\subsubsection{Discussion on Constraint Enforcement and Partial Observability}
A key challenge in this case study is that only the input current and temperature, and the output voltage are measured. The vast majority of the system's state (charges, concentrations, overpotentials) is latent. This scenario highlights a fundamental strength of the HRPINN architecture.
\begin{itemize}
    \item \textbf{HRPINN:} The recurrent structure is inherently suited for this task. The hidden state of the RNN naturally tracks the evolution of the full, unobserved physical state vector $\mathbf{x}$. The model is trained by comparing its final voltage computation with the measured data, and the error is backpropagated through time to correct the dynamics of the latent states. The explicit algebraic constraints are enforced by design within each forward pass of the recurrent cell, posing no additional difficulty.
    \item \textbf{NODE and PINN:} These baselines face a challenge. Lacking the architectural prior of the known physics, the NODE must infer the entire dynamics of the latent states from scratch. The PINN faces an additional challenge: it must produce a continuous-time trajectory for all latent states that simultaneously satisfies the soft-constrained physics loss while matching the sparse data loss on the voltage. This can lead to non-physical solutions for the latent states that happen to produce the correct output.
    \item \textbf{Framework Flexibility:} This case study demonstrates the flexibility of our proposed framework. The battery model is a semi-explicit DAE, a system whose explicit algebraic constraints are handled gracefully by the standard HRPINN architecture. Projection methods (i.e., PHRPINN), and by extension other projection-based models like PNODE, are not needed here and would be both inefficient and conceptually misaligned.
    Projection is designed to enforce implicit invariants on the differential states (e.g., energy conservation, $g(\mathbf{x})=0$). In this battery model, however, the algebraic constraints define the relationship between the differential states $\mathbf{x}$ and the algebraic states $\mathbf{z}$ ($g(\mathbf{x}, \mathbf{z}, w)=0$). The HRPINN's recurrent cell, acting as a numerical DAE solver, already computes the correct algebraic states at each time step to satisfy these constraints by construction. Applying an additional projection step would introduce significant computational overhead (solving a KKT system) to enforce constraints that are already respected by the model's design. This highlights that our framework correctly distinguishes between systems needing only hard-constrained dynamics (HRPINN) and those with implicit invariants that require projection for stability (PHRPINN), offering a tailored solution for different classes of physical systems.
\end{itemize}

\subsubsection{Utility for Prognostics}
Modeling the healthy behavior of a battery with high fidelity is a crucial first step for many real-world prognostic tasks. A robust model of the nominal system serves as a baseline against which to detect and quantify deviations caused by aging and degradation. For instance, this model can be directly used for anomaly detection, where the residual (the difference between the model's voltage prediction and the measured voltage) acts as a sensitive indicator of off-nominal behavior. While purely data-driven methods, such as autoencoders, are also powerful tools for this type of anomaly detection, our physics-informed approach provides a residual that is
potentially more interpretable, as deviations can be linked back to the underlying physical model. Furthermore, this residual is not just a simple error signal; it is a rich, informative feature that can be fed into downstream data-driven models for predicting Remaining Useful Life (RUL) or End of Life (EOL). By effectively filtering out the complex but healthy dynamics, the residual naturally highlights the subtle effects of degradation, making it a powerful input for machine learning algorithms tasked with long-term prediction.

\begin{table}[ht!]
\centering
\scriptsize 

\begin{minipage}[t]{0.54\textwidth}
    \vspace{0pt} 
    \caption{Notation and Definitions}
    \label{tab:paper_battery_notation_defs}
    \centering
    \begin{tabular}{|c|p{5.8cm}|}
    \hline
    \textbf{Symbol} & \textbf{Definition} \\ \hline
    $V_{U,i}$ & Equilibrium potential of electrode $i$ \newline ($n$: negative, $p$: positive) \\
    $U_0$ & Reference potential \\
    $R$ & Universal gas constant \\
    $T$ & Electrode temperature \\
    $n$ & Electrons transferred ($n=1$ for Li-ion) \\
    $F$ & Faraday's constant \\
    $x_i$ & Mole fraction of Li in electrode $i$ \\
    $V_{INT,i}$ & Activity correction term (Redlich-Kister) \\
    $N_i$ & Number of terms in R-K expansion \\
    $A_{i,k}$ & Fitting parameters in R-K expansion \\
    $J_i$ & Current density at electrode $i$ \\
    $\alpha$ & Symmetry factor \\
    $V_{\eta,i}$ & Overpotential at electrode $i$ \\
    $J_{i0}$ & Exchange current density at electrode $i$ \\
    $i_{app}$ & Applied electrical current \\
    $S_i$ & Surface area of electrode $i$ \\
    $k_i$ & Lumped parameter for electrode $i$ \\
    $x_{s,i}$ & Mole fraction of Li at surface \\
    $V$ & Overall battery voltage \\
    $V_o$ & Voltage drop due to internal resistance \\
    $R_o$ & Total internal resistance \\
    $q_i$ & Amount of Li ions in electrode $i$ (C) \\
    $q_{max}$ & Total available (mobile) Li ions (C) \\
    $C_{b,i}$ & Conc. of Li ions in bulk (mol/m$^3$) \\
    $C_{s,i}$ & Conc. of Li ions at surface (mol/m$^3$) \\
    $q_{b,i}$ & Amount of Li ions in bulk (mol) \\
    $q_{s,i}$ & Amount of Li ions at surface (mol) \\
    $v_{b,i}$ & Volume of bulk region (m$^3$) \\
    $v_{s,i}$ & Volume of surface region (m$^3$) \\
    $\dot{q}_{b,s,i}$ & Diffusion rate bulk-to-surface (mol/s) \\
    $D$ & Diffusion constant (m$^2$/s) \\
    $\tau_o$ & Time constant (internal resistance) \\
    $\tau_{\eta,i}$ & Time constant (overpotential) \\ \hline
    \end{tabular}
\end{minipage}%
\hfill 
\begin{minipage}[t]{0.44\textwidth}
    \vspace{0pt} 
    \caption{Battery Parameters (NASA PCoE)}
    \label{tab:paper_battery_params_values}
    \centering

    \textbf{(a) Redlich-Kister Expansion}
    \vspace{2pt}

    \begin{tabular}{ll}
    \hline
    \textbf{Param} & \textbf{Value} \\ \hline
    $U_{0,p}$ & 4.03 V \\
    $A_{p,0}$ & -33642.23 J/mol \\
    $A_{p,1}$ & 0.11 J/mol \\
    $A_{p,2}$ & 23506.89 J/mol \\
    $A_{p,3}$ & -74679.26 J/mol \\
    $A_{p,4}$ & 14359.34 J/mol \\
    $A_{p,5}$ & 307849.79 J/mol \\
    $A_{p,6}$ & 85053.13 J/mol \\
    $A_{p,7}$ & -1075148.06 J/mol \\
    $A_{p,8}$ & 2173.62 J/mol \\
    $A_{p,9}$ & 991586.68 J/mol \\
    $A_{p,10}$ & 283423.47 J/mol \\
    $A_{p,11}$ & -163020.34 J/mol \\
    $A_{p,12}$ & -470297.35 J/mol \\
    $U_{0,n}$ & 0.01 V \\
    $A_{n,0}$ & 86.19 J/mol \\ \hline
    \end{tabular}

    \vspace{0.5cm} 

    \textbf{(b) Physical Parameters}
    \vspace{2pt}

    \begin{tabular}{ll}
    \hline
    \textbf{Param} & \textbf{Value} \\ \hline
    $q_{\max}$ & $1.32 \times 10^4$ C \\
    $R$ & $8.314$ J/mol/K \\
    $F$ & $96487$ C/mol \\
    $n$ & $1$ \\
    $D$ & $7.0 \times 10^6$ mol s/C/m$^3$ \\
    $\tau_o$ & $10$ s \\
    $\alpha$ & $0.5$ \\
    $R_o$ & $0.085$ $\Omega$ \\
    $S_p$ & $2 \times 10^{-4}$ m$^2$ \\
    $k_p$ & $2 \times 10^4$ A/m$^2$ \\
    $v_{s,p}$ & $2 \times 10^{-6}$ m$^3$ \\
    $v_{b,p}$ & $2 \times 10^{-5}$ m$^3$ \\
    $\tau_{\eta,p}$ & $90$ s \\
    $S_n$ & $2 \times 10^{-4}$ m$^2$ \\
    $k_n$ & $2 \times 10^4$ A/m$^2$ \\
    $v_{s,n}$ & $2 \times 10^{-6}$ m$^3$ \\
    $v_{b,n}$ & $2 \times 10^{-5}$ m$^3$ \\
    $\tau_{\eta,n}$ & $90$ s \\ \hline
    \end{tabular}
\end{minipage}

\end{table}

\subsection{Case Study 2: Standard Constrained Benchmark Systems}
To demonstrate the generality of the PHRPINN architecture, we evaluate it on a suite of standard constrained dynamical systems. We adopt the benchmarks used in the PNODE paper \citep{pal2025semi} to provide a direct comparison and to gain deeper insights into the capabilities and trade-offs of our hybrid, residual-learning approach when combined with projection.

\subsubsection{Systems and Data}
We adopt the six benchmark systems from Pal et al. \citep{pal2025semi}, summarized in Table \ref{tab:benchmark_systems}. These systems cover a range of dynamics and constraint types, including the Lotka-Volterra, mass-spring, two-body, nonlinear spring, planar robot arm, and rigid body systems, and serve as a standard for evaluating constrained dynamics learning. Data is generated by integrating the ground-truth equations. To ensure a direct and fair comparison, we adopt the exact data generation parameters, integration settings, and evaluation horizons specified in the PNODE methodology \citep{pal2025semi}. A projected integration scheme is used to ensure all ground-truth trajectories strictly satisfy their respective algebraic invariants.

\subsubsection{Model Implementations}
For these benchmarks, we compare a comprehensive suite of models to systematically analyze the contributions of residual learning and hard constraint enforcement. To create the residual learning task for the grey-box models (HRPINN, PHRPINN), the full dynamics $\mathbf{f}$ of each benchmark were separated into a known component $\mathbf{f}_{\text{phys}}$ (the prior) and an unknown residual $\mathbf{f}_{\text{unk}}$ (the learning target). This separation, detailed in Table~\ref{tab:benchmark_priors}, follows a logical approach: for mechanical systems, we separate known kinematics from the unknown force-based dynamics, and for coupled systems, we separate the uncoupled terms from the coupling terms. This allows the grey-box models to leverage partial physical knowledge, while the black-box models (NODE, PNODE) must learn the entire dynamics $\mathbf{f} = \mathbf{f}_{\text{phys}} + \mathbf{f}_{\text{unk}}$ from data.
\begin{itemize}
    \item \textbf{NODE:} A black-box recurrent model that learns the entire dynamics function with a single neural network. It has no knowledge of the physics or the constraint.
    \item \textbf{PINN:} A standard baseline that uses a single MLP to represent the trajectory. It is trained with soft penalties for both the differential equations and the algebraic constraints.
    \item \textbf{HRPINN (Unconstrained):} A grey-box model that hard-codes the known part of the ODEs and learns the unknown residual. It is not made aware of the algebraic constraint.
    \item \textbf{HRPINN (Soft-Constrained):} The same grey-box HRPINN, but with an additional soft penalty term in its loss function to encourage satisfaction of the algebraic constraint.
    \item \textbf{PNODE:} A black-box model (identical to NODE) augmented with a projection layer that enforces the algebraic constraint after each integration step.
    \item \textbf{PHRPINN (Proposed):} Our proposed grey-box architecture that combines the residual learning of HRPINN with the hard-enforced projection layer of PNODE.
\end{itemize}

\subsubsection{Discussion}
The goal of this case study is to validate the PHRPINN framework as a robust and general-purpose tool. By comparing this full suite of models, we can systematically dissect the benefits of different modeling choices. The comparison between HRPINN and NODE isolates the advantage of residual learning; the comparison between soft-constrained HRPINN and PHRPINN highlights the effectiveness of hard projection over soft penalties; and the comparison between PHRPINN and PNODE demonstrates the value of combining residual learning with projection. Demonstrating strong performance across these diverse benchmarks validates that our proposed architecture is not narrowly tailored to a single problem but is a flexible tool applicable to a wide class of constrained dynamical systems.

\begin{table*}[ht!]
\centering
\caption{Benchmark constrained dynamical systems, adapted from \citep{pal2025semi}. Invariants are defined for total energy ($E$), angular momentum ($L$), a conserved Lotka-Volterra quantity ($V$), and general constraints ($C$). State variables $\mathbf{q}$ and $\mathbf{p}$ denote generalized position and momentum, respectively.}
\label{tab:benchmark_systems}
\small
\begin{tabular}{@{}lll@{}}
\toprule
\textbf{System} & \textbf{ODE System} & \textbf{Algebraic Invariant(s)} \\ \midrule
Lotka-Volterra & $\dot{x}=\alpha x-\beta x y$, $\dot{y}=\delta x y-\gamma y$ & $V(x,y)=\delta x-\gamma\ln x+\beta y-\alpha\ln y=V_0$ \\
Mass-Spring & $\dot{x}=v$, $\dot{v}=-x$ & $E(v,x)=\tfrac{1}{2}(x^2+v^2)=E_0$ \\
Two-Body & $\dot{\mathbf{q}}=\mathbf{p}$, $\dot{\mathbf{p}}=-\mathbf{q}/\|\mathbf{q}\|^3$ & $L(\mathbf{q},\mathbf{p})=q_1 p_2 - q_2 p_1 = L_0$ \\
Nonlinear Spring & $\dot{x}=u, \dot{u}=-x(x^2+y^2)$ & $E=\tfrac{1}{2}(u^2+v^2)+\tfrac{1}{4}(x^2+y^2)^2=E_0$ \\
& $\dot{y}=v, \dot{v}=-y(x^2+y^2)$ & $L=xv-yu=L_0$ \\
Planar Robot Arm & $\dot{\theta}=e'(\theta)^\top (e' e'^\top)^{-1}\dot{p}(t)$ & $C(\theta,t)=e(\theta)-p(t)=\mathbf{0}$ \\
Rigid Body & $\dot{\mathbf{y}} = \mathbf{y} \times (I^{-1}\mathbf{y})$ & $C(\mathbf{y})=\tfrac{1}{2}\|\mathbf{y}\|^2 = C_0$ \\ \bottomrule
\end{tabular}
\end{table*}

\begin{table*}[ht!]
\centering
\caption{Separation of full dynamics $\mathbf{f}$ into known physics (prior) $\mathbf{f}_{\text{phys}}$ and unknown residual $\mathbf{f}_{\text{unk}}$ for Case Study 2 benchmarks.}
\label{tab:benchmark_priors}
\small
\resizebox{\linewidth}{!}{%
\begin{tabular}{@{}lll@{}}
\toprule
\textbf{System} & \textbf{Known Physics (Prior) $\mathbf{f}_{\text{phys}}$} & \textbf{Unknown Residual (Learning Target) $\mathbf{f}_{\text{unk}}$} \\ \midrule

Lotka-Volterra & $\begin{aligned} \dot{x} &= \alpha x \\ \dot{y} &= -\gamma y \end{aligned}$ (Uncoupled terms) & $\begin{aligned} \dot{x} &= -\beta x y \\ \dot{y} &= \delta x y \end{aligned}$ (Coupling terms) \\ \addlinespace

Mass-Spring & $\begin{aligned} \dot{x} &= v \\ \dot{v} &= 0 \end{aligned}$ (Kinematics) & $\begin{aligned} \dot{x} &= 0 \\ \dot{v} &= -x \end{aligned}$ (Dynamics / Force) \\ \addlinespace

Two-Body & $\begin{aligned} \dot{\mathbf{q}} &= \mathbf{p} \\ \dot{\mathbf{p}} &= \mathbf{0} \end{aligned}$ (Kinematics) & $\begin{aligned} \dot{\mathbf{q}} &= \mathbf{0} \\ \dot{\mathbf{p}} &= -\mathbf{q}/\|\mathbf{q}\|^3 \end{aligned}$ (Dynamics / Force) \\ \addlinespace

Nonlinear Spring & $\begin{aligned} \dot{x} &= u, \dot{y} = v \\ \dot{u} &= 0, \dot{v} = 0 \end{aligned}$ (Kinematics) & $\begin{aligned} \dot{x} &= 0, \dot{y} = 0 \\ \dot{u} &= -x(x^2+y^2), \dot{v} = -y(x^2+y^2) \end{aligned}$ (Dynamics / Force) \\ \addlinespace

Planar Robot Arm & $\dot{\theta}=e'(\theta)^\top (e' e'^\top)^{-1}\dot{p}(t)$ (Full Kinematic Mapping) & $\mathbf{f}_{\text{unk}} = \mathbf{0}$ \\ \addlinespace

Rigid Body & $\mathbf{f}_{\text{phys}} = \mathbf{0}$ (No clear kinematic separation) & $\dot{\mathbf{y}} = \mathbf{y} \times (I^{-1}\mathbf{y})$ (Full Dynamic) \\ \bottomrule
\end{tabular}
}
\end{table*}

\section{Experiments and Results}
\label{sec:experiments_and_results}

This section empirically validates the HRPINN and PHRPINN frameworks, with experiments designed to systematically address the Validation Objectives (VOs) introduced in Section~\ref{sec:casestudies}. We first evaluate the core HRPINN on the real-world battery DAE (addressing VO1 \& VO2), followed by a comprehensive benchmark of the PHRPINN extension to assess its accuracy, physical consistency, and robustness (addressing VO3, VO4, \& VO5).

\subsection{General Experimental Setup}
This section outlines the common methodology used for training and evaluating all models to ensure a fair and reproducible comparison.

\subsubsection{Design and Principles}
Our experimental design is guided by fairness, simplicity, and reproducibility. All models are trained with a matched computational budget and a comparable number of trainable parameters (within $\pm10\%$). We intentionally use a common, fixed set of hyperparameters for all models rather than performing extensive tuning. While this may not reveal the peak performance of each architecture, we mitigate the risk of coincidental findings by reporting the mean and standard deviation over 25 runs. This provides a robust estimate of the expected performance and reveals the inherent inductive biases of the architectures themselves \footnote{The source code, datasets, and scripts to reproduce all experiments in this paper are available at: https://doi.org/10.5281/zenodo.17580787}.

A key methodological choice in our experiments was to use the raw, physically-scaled data directly for all models. We acknowledge that the normalization of inputs is a critical step for optimizing neural network training. While alternative physics-based normalization strategies exist, such as scaling variables by their known physical limits from component datasheets, its application in physics-informed contexts like PINNs remains complex; best practices often involve a multi-stage pipeline where network inputs are normalized but are then denormalized before being evaluated by the physics-based loss function to ensure the residuals are computed in their correct physical units.
Implementing such a model-specific preprocessing pipeline for the baselines would introduce a significant confounding variable when comparing against our HRPINN architecture, which is fundamentally designed to operate directly on physical values. To ensure a direct and unambiguous comparison focused purely on architectural merit, we adopted a policy of \textbf{methodological simplicity}: all models receive the exact same raw, unscaled data.

This approach not only ensures a direct comparison but also serves as a valuable stress test, evaluating each model's inherent robustness to the varying magnitudes of real-world physical quantities.
As detailed in Section \ref{sec:casestudies}, we compare our models against the Neural ODE (NODE), the Physics-Informed Neural Network (PINN), and the Projected Neural ODE (PNODE). PNODE is a recent baseline by Pal, Rackauckas et al.~\citep{pal2025semi} (from the group that introduced UDEs~\citep{rackauckas2020universal}) which augments a NODE with a projection layer to enforce hard constraints. We evaluate both its Fast and Robust projection variants. Our re-implementation of these baselines is intended for fair comparison within our specific experimental setup; therefore, all baseline results are reported to validate performance \textit{relative to our method} and are not a direct numerical replication of their original papers.

\subsubsection{Training Protocol and Evaluation}
Models were trained for up to 100 epochs for Case Study 1 (CS1) or 500 epochs for Case Study 2 (CS2) using the Adam optimizer with an initial learning rate of $\eta_0=1\mathrm{e}{-3}$ and a \texttt{ReduceLROnPlateau} scheduler. We assess performance on prediction accuracy using Mean Absolute Error (MAE) and Dynamic Time Warping (DTW), physical consistency via mean/max constraint violation, and computational cost (training time).
\subsection{Validating the Core HRPINN Architecture (VO1 \& VO2)}

To answer the first two research questions, we use the real-world battery prognostics case study, which provides a challenging testbed for the core HRPINN architecture on a partially observed DAE system.

\subsubsection{VO1: Explicit Invariant Handling in a Real-World DAE (CS1)}
\paragraph{Protocol.}
We task HRPINN, NODE, and PINN with predicting the terminal voltage curve. The hidden layer sizes were adjusted to ensure all models had a comparable number of trainable parameters ($\approx 10\text{k}$).

\paragraph{Results and Analysis.}
The results, summarized in Table~\ref{tab:cs1_results}, demonstrate a clear performance advantage for HRPINN. The parameter-matched HRPINN-Large model achieves an MAE of \textbf{0.0377}, a \textbf{66\% reduction} compared to the next-best PINN baseline. This level of accuracy is highly significant for prognostics, as a model with lower nominal error can detect smaller deviations, enabling more sensitive and earlier fault detection. This confirms that for a partially observed, semi-explicit DAE, HRPINN's architecture is significantly more effective. Furthermore, the HRPINN-Small model, with only \textbf{68 parameters}, outperforms the $\approx 10\text{k}$ parameter NODE, strongly supporting the data efficiency benefits of embedding known physics. Notably, the inference times for HRPINN-Large and Small are nearly identical; this is because the overall computational cost is dominated by the numerical solver for the hard-coded physics, making the overhead from the larger neural network negligible in comparison.

\begin{table*}[ht!]
\centering
\caption{Performance on Battery Prognostics (CS1). Results are mean $\pm$ std over 25 runs. Best performance is in \textbf{bold}. All models except HRPINN-Small have matched parameter counts ($\approx 10\text{k}$).}
\label{tab:cs1_results}
\resizebox{\linewidth}{!}{%
\begin{tabular}{@{}lccccr@{}}
\toprule
\textbf{Model} & \textbf{MAE (Voltage) $\downarrow$} & \textbf{DTW (Voltage) $\downarrow$} & \textbf{Training Time (s)} & \textbf{Inference Time/Sample (s) $\downarrow$} & \textbf{Parameters} \\ \midrule
\textbf{HRPINN-Large (ours)} & \textbf{0.0377 $\pm$ 0.0010} & \textbf{56.23 $\pm$ 12.95} & 1371.69 $\pm$ 368.29 & 3.5075 $\pm$ 0.4102 & 10,078 \\
HRPINN-Small (ours) & 0.1001 $\pm$ 0.0332 & 353.98 $\pm$ 314.20 & 1021.41 $\pm$ 301.54 & 3.3507 $\pm$ 0.0950 & \textbf{68} \\
PINN (Baseline) & 0.1127 $\pm$ 0.0265 & 458.27 $\pm$ 192.32 & \textbf{248.93 $\pm$ 92.60} & \textbf{0.0099 $\pm$ 0.0008} & 10,174 \\
NODE (Baseline) & 0.1458 $\pm$ 0.0024 & 1042.24 $\pm$ 99.39 & 1121.53 $\pm$ 227.22 & 0.1564 $\pm$ 0.0022 & 10,120 \\ \bottomrule
\end{tabular}%
}
\end{table*}

\subsection{Evaluating PHRPINN on Implicit Invariants (VO3, VO4, \& VO5)}

\subsubsection{Protocol for Standard Benchmarks (CS2)}
To address VO3--VO5, we use the \textbf{Standard Benchmarks (CS2)} for general validation. We compare PHRPINN against PNODE and PINN on long-horizon prediction tasks, evaluating trajectory accuracy, constraint violation, and computational cost. For the conjecture experiments (VO4), we performed 50 independent runs to assess optimization stability and generalization.

\subsubsection{VO3: Core Performance on Standard Benchmarks (CS2)}
\paragraph{Physical Consistency vs. Computational Cost.}
As summarized in Table~\ref{tab:CS2_fused}, the primary benefit of the predict-project architecture is evident in constraint satisfaction. Both \textbf{PHRPINN} and \textbf{PNODE} with fast projection consistently reduce mean constraint violations to the order of $10^{-7}$ (e.g., PNODE on "RigidBody": $1.70\mathrm{e}{-7}$), a vast improvement over soft-constrained models like PINN or a standard Neural ODE without projection, whose violations are orders of magnitude larger. However, this consistency comes at a high computational cost. While PINN models train in seconds (e.g., 5.2s on "MassSpring"), the robust projection variants can take hours (e.g., "PNODE-Robust" taking over 2.7 hours on "MassSpring"). This highlights a critical trade-off between guaranteed consistency and the computational cost of model development and retraining.

Notably, the PNODE-Fast baseline failed on the \texttt{NonlinearSpring} system, producing \texttt{NaN} values (Table~6). This empirical result validates the caveat from the PNODE authors, who warned that their "fast approximation" is not universally stable and that "there are cases where this model doesn't show promising outcomes" \citep{pal2025semi}.

\paragraph{Trajectory Accuracy and Shape Fidelity.}
While PNODE excels at enforcing constraints, \textbf{PHRPINN often achieves superior trajectory accuracy}, demonstrating the value of combining residual learning with projection. On the complex "RobotArm" system, PHRPINN-Fast achieves an MAE of $\mathbf{2.10 \times 10^{-3}}$, outperforming PNODE-Fast ($5.17 \times 10^{-3}$). Furthermore, PHRPINN's ability to capture the qualitative shape of trajectories is highlighted by its strong performance on the DTW metric for systems like "RigidBody" ($\mathbf{0.828}$ vs. PNODE's $1.39$). This suggests that by hard-coding known physics, the neural network can focus its capacity on learning the more subtle, unmodeled dynamics.

\paragraph{Ablation of HRPINN Architectures.}
Our results also provide an ablation study on the HRPINN framework itself. On the "RigidBody" system, a pure residual model with no constraints ("HRPINN\_NoConstraints") achieves a strong MAE of $2.41 \times 10^{-2}$ but has a high mean violation of $4.95 \times 10^{-1}$. Adding soft constraints ("HRPINN\_SoftConstraints") degrades accuracy (MAE $9.01 \times 10^{-2}$) while only marginally improving violations. Only by adding hard projection in the final \textbf{PHRPINN} model do we achieve both top-tier accuracy (MAE $\mathbf{1.86 \times 10^{-2}}$) and significantly improved constraint adherence, justifying the projected architecture.

\begin{table*}[ht!]
\centering
\caption{Comprehensive performance on Standard Benchmarks (CS2). We compare our proposed PHRPINN (Fast Projection) against key baselines across all six systems. Best performance among the three is in \textbf{bold}.}
\label{tab:CS2_fused}
\resizebox{\linewidth}{!}{%
\begin{tabular}{@{}l|ccc|ccc|ccc@{}}
\toprule
\multicolumn{1}{c}{} & \multicolumn{3}{c}{\textbf{PHRPINN-Fast (ours)}} & \multicolumn{3}{c}{\textbf{PNODE-Fast (Baseline)}} & \multicolumn{3}{c}{\textbf{PINN-Dynamics (Baseline)}} \\
\textbf{System} & MAE $\downarrow$ & DTW $\downarrow$ & Mean Viol. $\downarrow$ & MAE $\downarrow$ & DTW $\downarrow$ & Mean Viol. $\downarrow$ & MAE $\downarrow$ & DTW $\downarrow$ & Mean Viol. $\downarrow$ \\ \midrule
MassSpring & $1.10\mathrm{e}{-2}$ & $0.573$ & $5.07\mathrm{e}{-2}$ & $\mathbf{5.54\mathrm{e}{-3}}$ & $\mathbf{0.602}$ & $\mathbf{2.40\mathrm{e}{-7}}$ & $8.77\mathrm{e}{-3}$ & $0.960$ & $9.57\mathrm{e}{-2}$ \\
LotkaVolterra & $6.87\mathrm{e}{-2}$ & $1.38$ & $2.18\mathrm{e}{+0}$ & $\mathbf{6.89\mathrm{e}{-3}}$ & $\mathbf{0.776}$ & $\mathbf{1.90\mathrm{e}{-7}}$ & $9.45\mathrm{e}{-3}$ & $1.03$ & $4.11\mathrm{e}{+0}$ \\
TwoBody & $2.39\mathrm{e}{-1}$ & $38.5$ & $8.16\mathrm{e}{-1}$ & $\mathbf{2.22\mathrm{e}{-1}}$ & $\mathbf{32.1}$ & $\mathbf{1.67\mathrm{e}{-7}}$ & $2.94\mathrm{e}{-1}$ & $30.0$ & $1.43\mathrm{e}{+0}$ \\
NonlinearSpring & $\mathbf{9.69\mathrm{e}{-2}}$ & $\mathbf{16.5}$ & $6.35\mathrm{e}{-1}$ & $3.34\mathrm{e}{-1}$ & $57.7$ & NaN & $1.18\mathrm{e}{+0}$ & $186.0$ & $1.21\mathrm{e}{+0}$ \\
RobotArm & $\mathbf{2.10\mathrm{e}{-3}}$ & $\mathbf{0.311}$ & $1.95\mathrm{e}{+0}$ & $5.17\mathrm{e}{-3}$ & $0.764$ & $\mathbf{3.08\mathrm{e}{-7}}$ & $1.63\mathrm{e}{-1}$ & $18.9$ & $3.97\mathrm{e}{+0}$ \\
RigidBody & $\mathbf{1.86\mathrm{e}{-2}}$ & $\mathbf{0.828}$ & $5.00\mathrm{e}{-1}$ & $2.80\mathrm{e}{-2}$ & $1.39$ & $\mathbf{1.70\mathrm{e}{-7}}$ & $4.29\mathrm{e}{-2}$ & $3.31$ & $9.15\mathrm{e}{-1}$ \\
\bottomrule
\end{tabular}%
}
\end{table*}

\subsubsection{VO4: Framework Robustness and Conjectures (CS2)}

\paragraph{Framework Limitations.}
The stress tests reveal practical limits of the numerical methods employed, particularly when applied to stiff systems with unscaled data. On the \texttt{LotkaVolterra} system, both our PHRPINN-Robust model and our re-implemented PNODE-Robust baseline failed to produce stable results, yielding \texttt{NaN}s. This is a known challenge of this stiff benchmark. Indeed, the authors of PNODE noted that related methods, when tuned for high accuracy, can cause the "dynamics become stiff" and "warrant the use of slower implicit... solver"~\citep{pal2025semi}. We attribute our observed failures to this same underlying numerical instability, which is exacerbated in our setup by the use of explicit-in-time solvers on the \textbf{raw, unscaled physical data}. This finding does not contradict the benefit of physical priors but rather illustrates a crucial interaction: a strong, hard-coded physical model can expose a system's underlying numerical stiffness, which must be handled by a compatible numerical solver. This candidly demonstrates that the choice of model must be tailored to the system's numerical characteristics.

\paragraph{Conjecture 1: Improved Optimization.}
To test the conjecture that hard-coding physics improves the optimization landscape, we compare PHRPINN directly against PNODE. This comparison provides the most direct ablation, as both models use a hard projection layer, thus isolating the effect of the residual-learning prior on the optimization landscape. We analyzed the final training loss across 50 independent runs, and the results, summarized in Table~\ref{tab:conjecture1_full_improved}, strongly support this claim. For 5 out of the 6 benchmark systems, \textbf{PHRPINN achieves a lower mean final loss than PNODE}, often by an order of magnitude. Crucially, it also achieves lower variance and a higher convergence rate on systems like "RigidBody" (1.166 vs 0.990 for PNODE), indicating more reliable and efficient convergence. For instance, on "RobotArm", the standard deviation of PHRPINN's final loss is $\mathbf{4.53 \times 10^{-4}}$, nearly 7 times lower than PNODE's ($3.10 \times 10^{-3}$). This result provides strong empirical support for Conjecture 1, demonstrating that embedding known physics constrains the learning problem to a simpler, better-conditioned residual function that makes optimization both easier and more reliable.

\begin{table}[ht!]
\centering
\caption{Analysis of Optimization Landscape (Conjecture 1). We report the mean final training loss. Lower values suggest a better-conditioned optimization problem. Best performance is in \textbf{bold}.}
\label{tab:conjecture1_full_improved}
\resizebox{0.6\columnwidth}{!}{%
\begin{tabular}{@{} l S[table-format=1.2e-1]
                   S[table-format=1.2e-1]
                   S[table-format=1.2e-1] @{}}
\toprule
\textbf{System} & {\textbf{PHRPINN}} & {\textbf{PNODE}} & {\textbf{PINN}} \\ \midrule
RigidBody       & \bfseries 7.14e-5 & 1.78e-4 & 1.33e-2 \\
RobotArm        & \bfseries 5.64e-3 & 1.03e-1 & 1.48e-1 \\
LotkaVolterra   & \bfseries 3.56e-3 & 2.10e-2 & 2.61e-1 \\
MassSpring      & \bfseries 1.23e-5 & 3.07e-5 & 5.49e-4 \\
NonlinearSpring & {N/A}             & {N/A}   & \bfseries 2.78e-2 \\
TwoBody         & \bfseries 9.55e-3 & 5.62e-2 & 4.48e-2 \\
\bottomrule
\end{tabular}%
}
\end{table}

\paragraph{Conjecture 2: Superior Generalization.}
To test this conjecture, we compare PHRPINN against the standard PINN, as it represents a widely recognized and architecturally distinct (soft-constraint) approach. The results show that the conjecture holds, but its benefits are highly system-dependent. While the strong inductive bias can be a double-edged sword, the most favorable case is shown in Figure~\ref{fig:conjecture2_robotarm} for the complex \textbf{RobotArm} system, where \textbf{PHRPINN trained on only 10\% of the data achieves a lower MSE ($\mathbf{3.49 \times 10^{-7}}$) than a PINN trained on the full dataset ($4.70 \times 10^{-3}$) by several orders of magnitude}. This demonstrates a massive improvement in data efficiency, providing a clear, practical demonstration of the benefits hypothesized in Conjecture 2, where a lower-complexity residual learning task leads to tighter generalization from finite data.

However, this advantage is not universal, as detailed in the complete generalization results in Appendix \ref{app:Aappendix_generalization}, Figure~\ref{fig:appendix_generalization_all}. On simpler systems like "MassSpring", the standard PINN baseline generalized more effectively; its MSE improved by 94\% as data increased, whereas PHRPINN's error degraded significantly. This system-dependent behavior can be explained by analyzing the priors provided in Table~\ref{tab:benchmark_priors}.
For the \textbf{"RobotArm"} system, the "known" physics $\mathbf{f}_{\text{phys}}$ constituted the \textit{entire} complex kinematic mapping, leaving a trivial residual ($\mathbf{f}_{\text{unk}} = \mathbf{0}$) to be learned. This "perfect prior" made the learning task exceptionally simple for PHRPINN, resulting in the massive data efficiency gains observed.
Conversely, for the \textbf{"MassSpring"} system, the prior $\mathbf{f}_{\text{phys}} = [v, 0]$ was minimal, leaving the \textit{entire} system dynamic (the restoring force $\mathbf{f}_{\text{unk}} = [0, -x]$) in the residual. Because the true system is so simple, the structural bias of the HRPINN architecture offered no advantage over the flexibility of a standard PINN; in fact, it appears to have hindered generalization.
This highlights a key insight: the benefit of a physical prior is not just in its presence, but in its \textit{quality} and its ability to reduce the complexity of the residual function that the neural network must learn.

\begin{figure}[ht!]
\centering
\includegraphics[width=0.8\columnwidth]{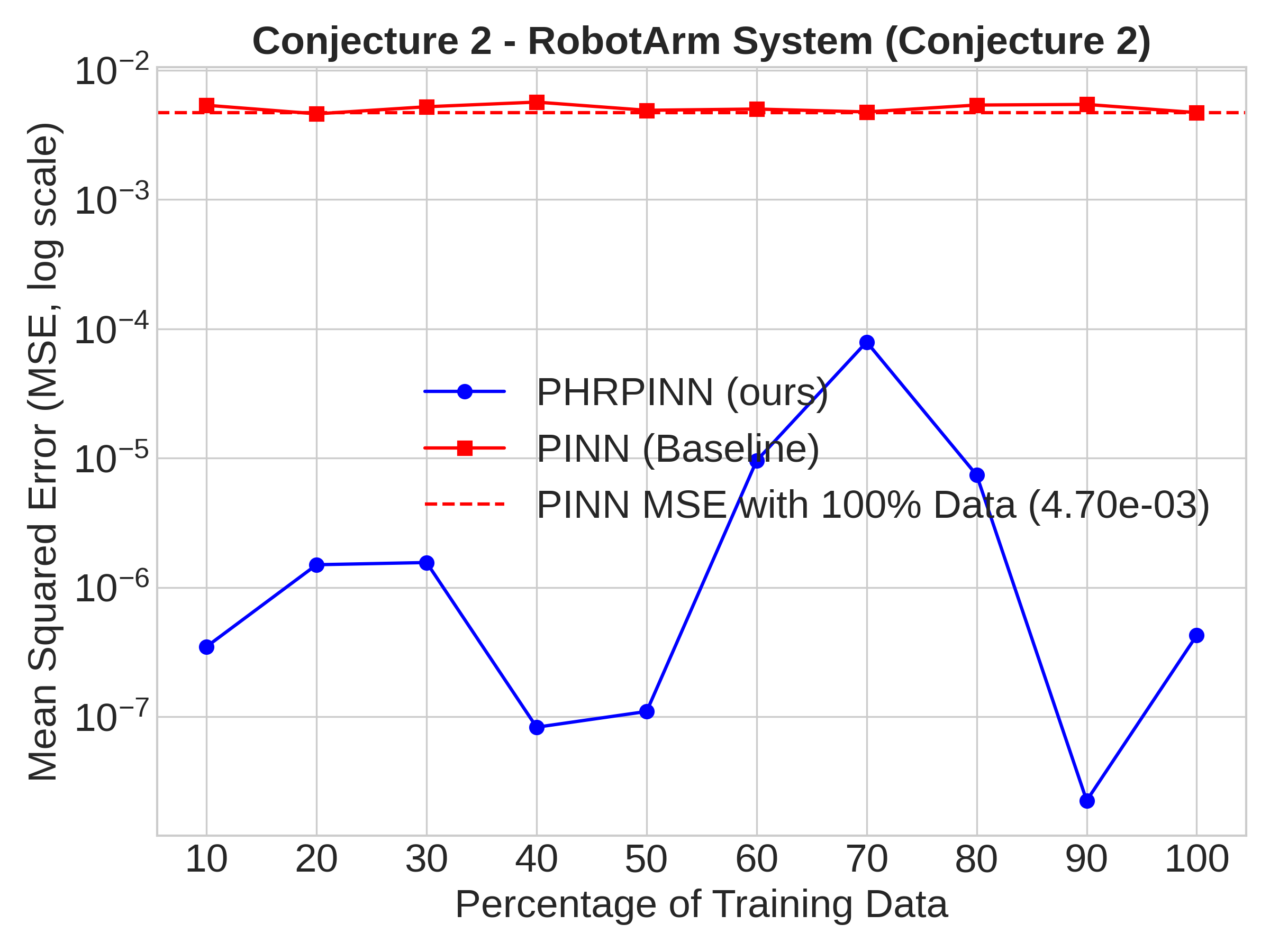}
\caption{Generalization learning curves for the "RobotArm" system (Conjecture 2). PHRPINN (blue) achieves significantly lower MSE than PINN (red), even with a fraction of the training data. The fluctuations in the PHRPINN curve, while appearing large on the logarithmic scale, are characteristic of the training process, arising from the stochastic nature of mini-batch sampling and optimizer dynamics.}
\label{fig:conjecture2_robotarm}
\end{figure}

\section{Discussion and Practical Deployment}

\label{sec:discussion}
The experimental results presented in this paper provide direct insights into the research questions and conjectures posed at the outset. Our findings suggest that the proposed framework, in its HRPINN and PHRPINN variants, offers a viable approach for balancing model fidelity with physical consistency in different classes of dynamical systems. In Case Study 1, the HRPINN architecture demonstrated its effectiveness on a real-world, partially observed DAE, directly addressing \textbf{VO1}. The model's ability to outperform a 150x
larger black-box model (HRPINN-Small vs. NODE) provides a clear
answer to \textbf{VO2}, confirming the data and parameter efficiency of the residual learning approach. This result also provides practical support for \textbf{Conjecture~IV.2}. The 68-parameter HRPINN-Small model, which embeds the known electrochemical laws ($\mathbf{f}_{\mathrm{phys}}$), was tasked with learning only the relatively simple non-ideal thermodynamic residual ($\hat{V}_{INT,i}$). Its ability to outperform a $\approx 10\text{k}$-parameter black-box NODE demonstrates the conjecture's premise: when the known physics are significant and the residual is simple, the hard-constrained approach yields superior data and parameter efficiency. Subsequently, the benchmarks in Case Study 2 answered \textbf{VO3}, showing that the PHRPINN extension successfully enforces implicit invariants to a high degree of precision, albeit with a notable tradeoff between computational cost and the robustness of the projection method.

Our investigation also yielded strong empirical support for the conjectures on the benefits of the hard-constrained architecture (\textbf{VO4}). The results summarized in Table~\ref{tab:conjecture1_full_improved} support \textbf{Conjecture~1}, indicating that by learning a simpler residual function, PHRPINN fosters a more stable and better-conditioned optimization landscape than its black-box counterpart. The evidence for \textbf{Conjecture~2} was more nuanced; while PHRPINN showed a remarkable improvement in generalization on the complex "RobotArm" system (Figure~\ref{fig:conjecture2_robotarm}), its strong inductive bias was less effective on simpler systems, highlighting the system-dependent nature of this benefit. Finally, the ablation experiments on the "RigidBody" system provided a definitive answer to \textbf{VO5}: simple soft-constraint penalties were insufficient for ensuring physical consistency, making the hard projection in PHRPINN a necessary component for achieving both high accuracy and low constraint violation. We re-emphasize that the numerical instabilities observed in some stress tests were carefully investigated and are characteristic of the underlying system dynamics interacting with the chosen solver, not implementation flaws, reinforcing the importance of tailoring the numerical approach to the problem's nature.

These findings inform a more nuanced approach to practical deployment, centered on the entire model lifecycle. \textbf{1) For development,} the high computational cost of the PHRPINN-Robust variant makes it most suitable for creating high-fidelity, offline digital twins where correctness is paramount and development time is less critical. \textbf{2) For deployment,} where inference time is the key metric, the PHRPINN-Fast variant offers a compelling balance for online monitoring or control, providing physical consistency orders of magnitude better than soft methods with a manageable computational load. \textbf{3) For retraining,} systems that require frequent updates may favor simpler architectures like a soft-constrained HRPINN to minimize recurring computational costs, unless the strict guarantees of projection are a hard requirement. While these results validate the framework's potential, our study also illuminated key limitations and clear avenues for future research, which we address next.

\section{Limitations and Future Work}
\label{sec:limitations}
While our results validate the core strengths of the HRPINN/PHRPINN framework, this study also illuminated its limitations, which define clear avenues for future research.

\paragraph{Numerical Stability and Stiffness}
The investigation into \textbf{VO4} revealed that the hard-constrained projection in PHRPINN is not a panacea. We observed numerical instability on certain benchmarks (e.g., NaNs in \texttt{LotkaVolterra}), particularly with the robust solver. This suggests that the projection step, especially when correcting for large prediction errors, can introduce stiffness into the system dynamics. Future work should explore mitigation strategies such as incorporating more advanced or implicit numerical integrators (e.g., Radau or BDF methods) into the recurrent cell, or implementing adaptive projection techniques that blend fast and robust solvers. A key part of this would be developing a lightweight, quantitative criterion (e.g., based on an estimate of the constraint manifold's curvature $||H||$ or the projection residual) to automatically determine when the 'fast' tangent-space approximation is insufficient and the 'robust' solver is required.

\paragraph{Scalability to Large-Scale Systems}
The projection step's computational cost, particularly the $\mathcal{O}((n+m)^3)$ complexity of solving and differentiating through the dense KKT system, poses a potential bottleneck for systems with a very large number of states ($n$) or constraints ($m$). While the tangent-space approximation improves this, scalability remains a challenge. A crucial research direction is to integrate sparse linear algebra and iterative solvers (e.g., Krylov subspace methods) into the differentiable projection layer to efficiently handle large, sparse Jacobians common in fields like finite element analysis or multibody dynamics.

\paragraph{Model Misspecification and Generalization}
As discussed, the framework's performance, and particularly the generalization benefit of residual learning, depends on the known physics $\mathbf{f}_{\mathrm{phys}}$ being a reasonably accurate representation of the system.
The nuanced results for Conjecture 2 highlight that a strong but incorrect inductive bias can be detrimental.
This benefit is also not universal; the conjecture's premise may fail in systems where the unmodeled residual dynamics $\mathbf{f}_{\mathrm{unk}}$ are intrinsically more complex or chaotic than the system's overall observed response. In such cases, where the system's known dynamics effectively filter these complex residuals into a simpler trajectory, learning the full trajectory $\mathbf{x}(t)$ may be an easier task.
This detrimental effect is most pronounced when the structural form of the physical model is incorrect (e.g., omitting a crucial physical effect), as opposed to simply having uncertain parameters.
Future research could explore hybrid approaches where parts of the "known" physics are parameterized and learned, allowing the model to correct for uncertainty or errors in the first-principles model.

\paragraph{Future Research Directions}
Building on these points, we prioritize the following extensions:
\begin{itemize}
    \item \textbf{Higher-Index DAEs:} Adapting the framework to handle higher-index DAEs, which are common in mechanical systems. This would likely require moving beyond simple projection to techniques based on index reduction or stabilized formulations.
    \item \textbf{Exhaustive Benchmarking:} Situating the framework's performance more precisely within the rapidly evolving PIML landscape through rigorous benchmarking against a wider range of architectures, including modern HNN/LNN variants and ALM-based methods.
    \item \textbf{Industrial-Scale Validation:} Validating the framework on large-scale industrial datasets to test its robustness and scalability under real-world conditions of noise, partial observability, and model uncertainty.
    \item \textbf{Knowledge Discovery:} Beyond prediction, using the learned residual function $\hat{\mathbf{f}}_{\boldsymbol{\theta}}$ as a high-integrity model of physical discrepancies. This opens avenues for downstream analysis with techniques like symbolic regression to uncover the analytical forms of unmodeled physics and contribute to scientific understanding.
    \item \textbf{Formal Analysis of Conjectures:} While Conjectures IV.1 (optimization) and IV.2 (generalization) are empirically supported, they lack rigorous proof. Future work could involve a formal analysis of the loss landscape's properties (e.g., the Hessian or condition number) to prove the conditions under which HRPINN is better conditioned. Similarly, a rigorous proof for Conjecture IV.2 would require formally characterizing the conditions under which the residual function class $\mathcal{F}_{unk}$ has a provably lower Rademacher complexity than the full solution class $\mathcal{F}_{traj}$.
\end{itemize}

\section{Conclusion}
\label{sec:conclusion}
To address the identified gap for a framework that embeds known physics, learns only residuals, and strictly enforces algebraic invariants, we formalized HRPINN and introduced its projected extension, PHRPINN. This work formalized their architecture, provided a proof of representational equivalence, and empirically evaluated their performance against established physics-informed baselines. Our experiments demonstrated that the architecture is effective and data-efficient for real-world DAEs (\textbf{VO1}, \textbf{VO2}), can strictly enforce implicit invariants with manageable tradeoffs (\textbf{VO3}, \textbf{VO5}), and exhibits favorable optimization properties (\textbf{VO4}). By presenting and validating this methodology, this work contributes a new, principled approach to the PIML toolkit. We believe this hard-constrained, residual-learning paradigm offers a promising and principled path toward developing the dependable and physically consistent digital twins required for the next generation of safety-critical cyber-physical systems.


\acks{The source code, datasets, and scripts to reproduce all experiments in this paper are available at: \url{https://doi.org/10.5281/zenodo.17580787}. The NASA PCoE Li-ion battery collection used in Case Study 1 is a public dataset available from the NASA PCoE data set repository. This work was supported by FUNDEP under Grant Rota 2030/Linha VI 29271.02.01/2022.01-00; and Grant Rota 2030/Linha VI 29271.03.01/2023.04-00. The authors report there are no competing interests to declare.
}


\newpage



\appendix
\section{Comprehensive Generalization Results}
\label{app:Aappendix_generalization}

This section provides a detailed, system-by-system analysis of the generalization results shown in Figure~\ref{fig:appendix_generalization_all}, which support the discussion of \textbf{Conjecture 2} (Superior Generalization) in the main text. The plots visualize the test Mean Squared Error (MSE) of the proposed \textbf{PHRPINN} against the baseline \textbf{PINN} as the percentage of available training data increases.

The results confirm that the generalization benefit from PHRPINN's strong inductive bias is highly \textbf{system-dependent}. A breakdown of the performance on each benchmark system is provided below.

For the \textbf{RobotArm} system, PHRPINN demonstrates a significant and consistent advantage. Its MSE remains several orders of magnitude lower than the PINN baseline across all training data percentages. This represents the strongest case for improved data efficiency, where hard-coding the known physical model is highly effective.

Conversely, the \textbf{MassSpring} system shows the opposite trend. The PINN baseline consistently outperforms PHRPINN, and its error generally decreases as more training data is provided. PHRPINN's error is both higher and more erratic, suggesting that its strong inductive bias may be a hindrance for this particular system's dynamics.

On the \textbf{RigidBody} benchmark, the results are mixed. While PHRPINN achieves a lower error with very sparse data (20\%), the PINN baseline shows more stable performance and a slightly lower average error across the majority of the data splits. Both models exhibit considerable variance in their performance as data size changes.

Similar to the MassSpring case, the PINN baseline shows a clear performance advantage on the \textbf{LotkaVolterra} system. Its MSE is consistently lower than that of PHRPINN, which does not appear to benefit from additional training data and fails to match the baseline's accuracy.

The \textbf{NonlinearSpring} system presents the most ambiguous results. The performance of both models is highly variable, with their learning curves intersecting multiple times. Neither model establishes a consistent advantage, indicating that both struggle with the system's dynamics or that the generalization behavior is particularly sensitive to the specific training data subsets.

For the \textbf{TwoBody} problem, PINN maintains a modest but consistent advantage over the PHRPINN. While both models operate within a similar error range, PINN's MSE is generally lower across all data percentages.

\begin{figure*}[ht!]
\centering
\includegraphics[width=1.0\textwidth]{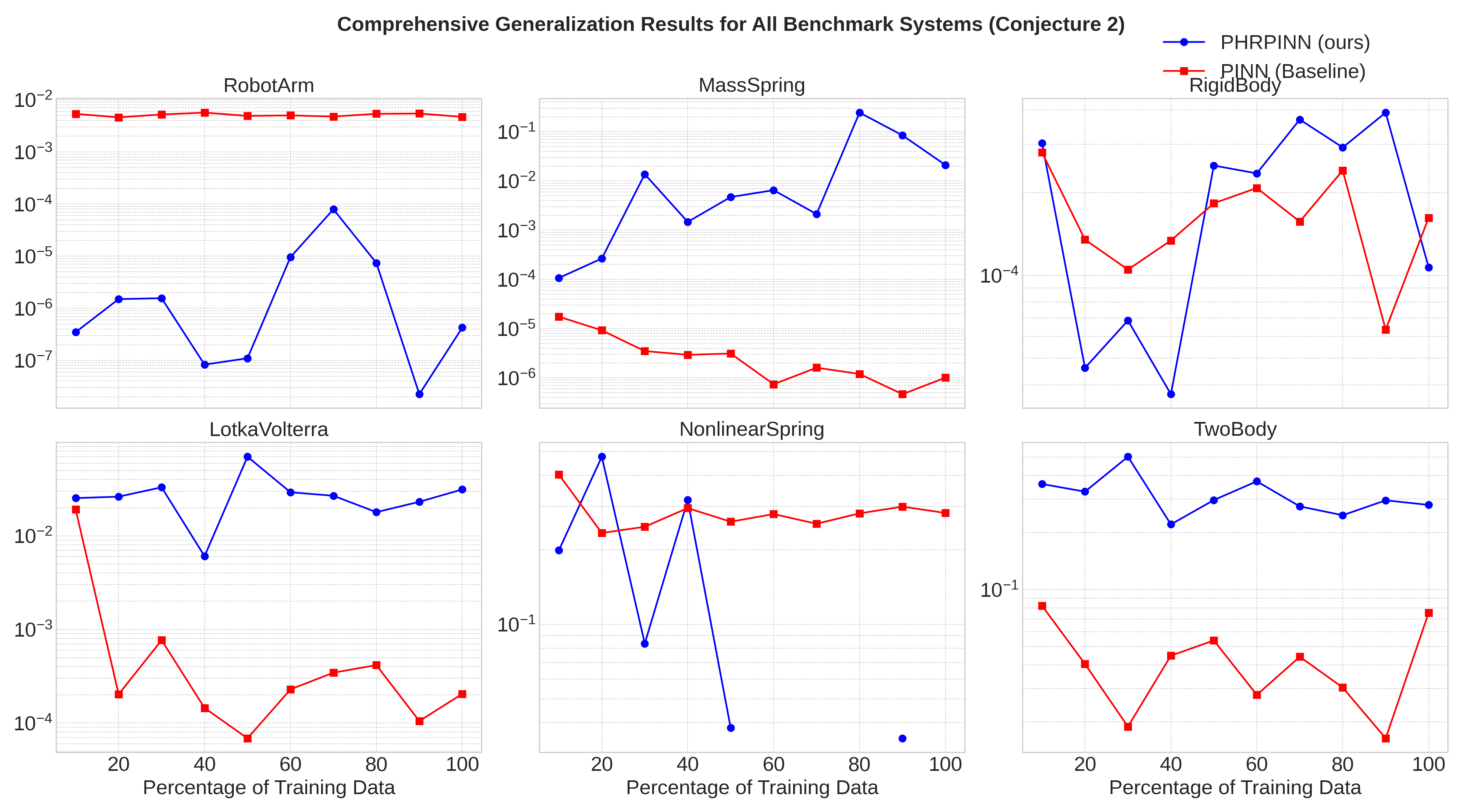}
\caption{Comprehensive generalization learning curves for all six benchmark systems, supporting the discussion for Conjecture 2. Each subplot compares the Mean Squared Error (log scale) of PHRPINN (blue) and PINN (red) as the percentage of available training data increases. These results highlight the system-dependent nature of the generalization advantage conferred by PHRPINN's strong inductive bias, showing strong outperformance on "RobotArm" but underperformance on simpler systems like "MassSpring".}
\label{fig:appendix_generalization_all}
\end{figure*}

\section{Derivations for the Differentiable Projection Operator}
\label{app:Bprojection_derivation}

This appendix provides the detailed mathematical derivations for the projection operator $\Pi$ and its Jacobian, which are summarized in Section~\ref{sec:method_overview}.

\subsection{Derivation of the KKT Sensitivity System (Eq. 6)}

The projection $\mathbf{x}^* = \Pi(\tilde{\mathbf{x}})$ is the solution to the optimization problem:
\begin{equation}
 \mathbf{x}^* = \arg\min_{\mathbf{x}} \frac{1}{2} ||\mathbf{x} - \tilde{\mathbf{x}}||^2 \quad \text{s.t.} \quad \mathbf{g}(\mathbf{x}) = \mathbf{0}.
\end{equation}
The Karush-Kuhn-Tucker (KKT) conditions (our Eq. \ref{eq:phrpinn-kkt}) form a root-finding problem $F(\mathbf{x}^*, \boldsymbol{\lambda}, \tilde{\mathbf{x}}) = \mathbf{0}$:
\begin{equation} \label{eq:app_kkt_conditions}
 F(\mathbf{x}^*, \boldsymbol{\lambda}, \tilde{\mathbf{x}}) = \begin{cases} \mathbf{x}^* - \tilde{\mathbf{x}} + G(\mathbf{x}^*)^\top \boldsymbol{\lambda} = \mathbf{0} \\ \mathbf{g}(\mathbf{x}^*) = \mathbf{0} \end{cases}
\end{equation}
where $G(\mathbf{x}) = \partial \mathbf{g} / \partial \mathbf{x}$ is the constraint Jacobian.

To find the sensitivities ($d\mathbf{x}^*, d\boldsymbol{\lambda}$) with respect to a perturbation $d\tilde{\mathbf{x}}$, we apply the implicit function theorem, as detailed in foundational texts like Nocedal \& Wright \cite{nocedal2006numerical}. The theorem states that if the Jacobian of $F$ with respect to $(\mathbf{x}^*, \boldsymbol{\lambda})$ is non-singular, then:
\begin{equation*}
 \frac{\partial F}{\partial (\mathbf{x}^*, \boldsymbol{\lambda})} \begin{bmatrix} d\mathbf{x}^* \\ d\boldsymbol{\lambda} \end{bmatrix} + \frac{\partial F}{\partial \tilde{\mathbf{x}}} d\tilde{\mathbf{x}} = \mathbf{0}.
\end{equation*}
We compute the necessary partial Jacobians from Equation \eqref{eq:app_kkt_conditions}:

\begin{align*}
 \text{Let } F_1 &= \mathbf{x}^* - \tilde{\mathbf{x}} + G(\mathbf{x}^*)^\top \boldsymbol{\lambda} \\
 \text{Let } F_2 &= \mathbf{g}(\mathbf{x}^*) \\
 \text{Let } H_L &= I + \sum_i \lambda_i \nabla^2 g_i(\mathbf{x}^*) \\
 \\
 \frac{\partial F}{\partial (\mathbf{x}^*, \boldsymbol{\lambda})} &= \begin{bmatrix}
 \frac{\partial F_1}{\partial \mathbf{x}^*} & \frac{\partial F_1}{\partial \boldsymbol{\lambda}} \\
 \frac{\partial F_2}{\partial \mathbf{x}^*} & \frac{\partial F_2}{\partial \boldsymbol{\lambda}}
 \end{bmatrix}
 = \begin{bmatrix}
 H_L & G(\mathbf{x}^*)^\top \\
 G(\mathbf{x}^*) & 0
 \end{bmatrix} \\
 \\
 \frac{\partial F}{\partial \tilde{\mathbf{x}}} &= \begin{bmatrix} \frac{\partial F_1}{\partial \tilde{\mathbf{x}}} \\ \frac{\partial F_2}{\partial \tilde{\mathbf{x}}} \end{bmatrix} = \begin{bmatrix} -I \\ 0 \end{bmatrix}
\end{align*}

By differentiating the optimality conditions $F(\mathbf{x}^*, \boldsymbol{\lambda}, \tilde{\mathbf{x}}) = \mathbf{0}$ with respect to $\tilde{\mathbf{x}}$ and rearranging terms, we obtain the full KKT sensitivity system:
\begin{equation}
 \begin{bmatrix}
 I + \sum_i \lambda_i \nabla^2 g_i(\mathbf{x}^*) & G(\mathbf{x}^*)^\top \\
 G(\mathbf{x}^*) & 0
 \end{bmatrix}
 \begin{bmatrix} d\mathbf{x}^* \\ d\boldsymbol{\lambda} \end{bmatrix}
 \;=\;
 \begin{bmatrix} d\tilde{\mathbf{x}} \\ 0 \end{bmatrix}.
\label{eq:app_kkt_linearized}
\end{equation}
This linear system defines the "robust" Jacobian $J_\Pi^{\mathrm{exact}}$ of the projection (specifically, $J_\Pi^{\mathrm{exact}} = \frac{d\mathbf{x}^*}{d\tilde{\mathbf{x}}}$ is the top-left block of the solution to this system).

\begin{lemma}[Differentiability of the orthogonal projection]
\label{lem:app_proj_diff}
Assume A6 (LICQ) and A6$'$ (SOSC / local non-degeneracy) hold at the orthogonal projection $\mathbf{x}^*$ of $\tilde{\mathbf{x}}$. Then the bordered KKT Jacobian in \eqref{eq:app_kkt_linearized} is nonsingular at $(\mathbf{x}^*,\boldsymbol{\lambda}^*)$. Consequently, by the implicit function theorem the projection
mapping $\Pi:\tilde{\mathbf{x}} \mapsto \mathbf{x}^*$ is $C^1$ in a neighborhood of $\tilde{\mathbf{x}}$, and its Jacobian
is given by the solution operator of the linear system \eqref{eq:app_kkt_linearized}.
\end{lemma}

\subsection{Justification for the Tangent-Space Projector (Eq. 7)}

The "fast" tangent-space projector is derived by making a principled approximation to the full sensitivity system \eqref{eq:app_kkt_linearized}.

\paragraph{Derivation}
We neglect the Hessian term, $H := \sum_i \lambda_i \nabla^2 g_i(\mathbf{x}^*)$. This is a well-founded approximation analogous to the one used in the Gauss-Newton algorithm for nonlinear least squares. It is valid when the constraint curvature (weighted by the multipliers) is small. Setting $H=0$, the sensitivity system \eqref{eq:app_kkt_linearized} simplifies to:
\begin{equation*}
 \begin{bmatrix} I & G(\mathbf{x}^*)^\top \\ G(\mathbf{x}^*) & 0 \end{bmatrix}
 \begin{bmatrix} d \mathbf{x}^* \\ d\boldsymbol{\lambda} \end{bmatrix}
 =
 \begin{bmatrix} d\tilde{\mathbf{x}} \\ 0 \end{bmatrix}
\end{equation*}
This represents two linear equations:
\begin{align}
 d\mathbf{x}^* + G(\mathbf{x}^*)^\top d\boldsymbol{\lambda} &= d\tilde{\mathbf{x}} \label{eq:app_alg_a} \\
 G(\mathbf{x}^*) d\mathbf{x}^* &= \mathbf{0} \label{eq:app_alg_b}
\end{align}
From \eqref{eq:app_alg_a}, we isolate $d\mathbf{x}^*$: $d\mathbf{x}^* = d\tilde{\mathbf{x}} - G(\mathbf{x}^*)^\top d\boldsymbol{\lambda}$. Substituting this into \eqref{eq:app_alg_b}:

\begin{align}
 G(\mathbf{x}^*) (d\tilde{\mathbf{x}} - G(\mathbf{x}^*)^\top d\boldsymbol{\lambda}) = \mathbf{0} \implies \\ G(\mathbf{x}^*) d\tilde{\mathbf{x}} = G(\mathbf{x}^*)G(\mathbf{x}^*)^\top d\boldsymbol{\lambda}
\end{align}

Assuming $G(\mathbf{x}^*)$ has full row rank (our LICQ assumption A6), $G(\mathbf{x}^*)G(\mathbf{x}^*)^\top$ is invertible. We solve for $d\boldsymbol{\lambda}$:
\begin{equation*}
 d\boldsymbol{\lambda} = (G(\mathbf{x}^*) G(\mathbf{x}^*)^\top)^{-1} G(\mathbf{x}^*) d\tilde{\mathbf{x}}
\end{equation*}
Finally, substituting this expression for $d\boldsymbol{\lambda}$ back into the equation for $d\mathbf{x}^*$:
\begin{align*}
 d\mathbf{x}^* &= d\tilde{\mathbf{x}} - G(\mathbf{x}^*)^\top \left( (G(\mathbf{x}^*) G(\mathbf{x}^*)^\top)^{-1} G(\mathbf{x}^*) d\tilde{\mathbf{x}} \right) \\
 &= \left[ I - G(\mathbf{x}^*)^\top (G(\mathbf{x}^*) G(\mathbf{x}^*)^\top)^{-1} G(\mathbf{x}^*) \right] d\tilde{\mathbf{x}}
\end{align*}
This shows that the Jacobian of the approximated projection is $J_\Pi^{\mathrm{tan}} = I - G^\top (GG^\top)^{-1} G$, which is exactly Equation \eqref{eq:tangent_projector}.

\paragraph{Theoretical Validity}
This tangent-space formula is algebraically exact when the constraint function $\mathbf{g}$ is affine (linear). For nonlinear constraints, the error of this approximation is bounded. Treating the Hessian $H$ as a perturbation, if the bordered KKT Jacobian in \eqref{eq:app_kkt_linearized} is invertible with minimum singular value $\sigma_{\min}>0$ and $\|H\|_{\mathrm{op}}\le\eta$, then standard matrix perturbation estimates (e.g., Stewart \& Sun \cite{stewart1990matrix}) imply:
\begin{equation*}
 \big\|J_\Pi^{\mathrm{exact}} - J_\Pi^{\mathrm{tan}}\big\| = O\!\left(\frac{\eta}{\sigma_{\min}}\right).
\end{equation*}
This confirms the tangent-space formula is a robust approximation when constraint curvature (weighted by multipliers) is small relative to the KKT system's conditioning.

\section{Proof of Theorem IV.3 (Representational Equivalence)}
\label{app:Ctheoretical_foundations}

This appendix provides the detailed theoretical underpinnings for the representational equivalence between standard Physics-Informed Neural Networks (PINNs) with physics-based regularization and the proposed Hybrid Recurrent PINN (HRPINN) architecture.

\subsection{Preliminaries and Formulations}
Let the underlying physical system be governed by a system of Ordinary Differential Equations (ODEs):
\begin{equation}
\label{eq:appendix_true_ode_system}
\dot{\mathbf{x}}(t) = \mathbf{F}(\mathbf{x}(t), t) = \mathbf{f}_{\text{phys}}(\mathbf{x}(t), t) + \mathbf{f}_{\mathrm{unk}}(\mathbf{x}(t), t), \quad \mathbf{x}(0) = \mathbf{x}_0,
\end{equation}
where $t \in [0,T]$, $\mathbf{x}(t) \in \mathbb{R}^n$ is the true state vector, $\mathbf{f}_{\text{phys}}$ is the known part of the dynamics, and $\mathbf{f}_{\mathrm{unk}}$ is the unknown part. We denote the true solution as $\mathbf{x}^*(t)$.

\paragraph{Standard PINN Formulation}
A standard PINN approximates the true solution $\mathbf{x}^*(t)$ with a neural network $\hat{\mathbf{x}}_{\boldsymbol{\theta}}(t)$ and the unknown physics with another network $\hat{\mathbf{f}}_{\text{unk}, \boldsymbol{\psi}}(\hat{\mathbf{x}}_{\boldsymbol{\theta}}(t), t)$. The physics-informed residual is:
\begin{equation}
\label{eq:appendix_std_pinn_residual}
\mathcal{N}[\hat{\mathbf{x}}_{\boldsymbol{\theta}}, \hat{\mathbf{f}}_{\text{unk}, \boldsymbol{\psi}}](t) \equiv \dot{\hat{\mathbf{x}}}_{\boldsymbol{\theta}}(t) - \mathbf{f}_{\text{phys}}(\hat{\mathbf{x}}_{\boldsymbol{\theta}}(t), t) - \hat{\mathbf{f}}_{\text{unk}, \boldsymbol{\psi}}(\hat{\mathbf{x}}_{\boldsymbol{\theta}}(t), t).
\end{equation}
The training loss, $\mathcal{L}_{\text{PINN}}$, combines a data-fitting term with a penalty on this residual.

\paragraph{HRPINN Formulation}
The HRPINN discretizes Eq. \eqref{eq:appendix_true_ode_system} using a time step $\Delta t$. Its state $\mathbf{h}_k \in \mathbb{R}^n$ at time $t_k = k\Delta t$ approximates $\mathbf{x}^*(t_k)$. The unknown component is learned by a neural network $\hat{\mathbf{f}}_{\boldsymbol{\phi}}(\mathbf{h}_k, t_k)$. The state update rule (using Forward Euler for this analysis) is:
\begin{equation}
\label{eq:appendix_hrpinn_update}
\mathbf{h}_{k+1} = \mathbf{h}_k + \Delta t \left( \mathbf{f}_{\text{phys}}(\mathbf{h}_k, t_k) + \hat{\mathbf{f}}_{\boldsymbol{\phi}}(\mathbf{h}_k, t_k) \right).
\end{equation}
The loss, $\mathcal{L}_{\text{HRPINN}}$, compares the unrolled trajectory $\{\mathbf{h}_k\}$ to observed data.

\subsection{Proof of Representational Equivalence}
The proof relies on the assumptions A1-A6 stated in the main text, which guarantee well-posedness, regularity, and the applicability of the Universal Approximation Theorem (UAT). We restate the theorem for clarity.

\begin{theorem*}[Representational Equivalence]
(Restated from Theorem~\ref{thm:equivalence})
\end{theorem*}

\begin{proof}
The proof uses the properties of the numerical integrator and the UAT. For this analysis, we use the Forward Euler method, whose Local Truncation Error (LTE) for the true solution $\mathbf{x}^*(t)$ is key. By Taylor's theorem, the true solution satisfies:
\[
\mathbf{x}^*(t_{k+1}) = \mathbf{x}^*(t_k) + \Delta t \cdot \mathbf{F}(\mathbf{x}^*(t_k),t_k) + \mathbf{LTE}_{k+1},
\]
where $\|\mathbf{LTE}_{k+1}\| \le \frac{\Delta t^2}{2} M_2$ for some constant $M_2$ bounding the second derivative of $\mathbf{x}^*(t)$.

\paragraph{Part 1: (PINN $\Rightarrow$ HRPINN)}
The premise is that a standard PINN, by the UAT, can learn networks $\hat{\mathbf{x}}_{\boldsymbol{\theta}}$ and $\hat{\mathbf{f}}_{\text{unk}, \boldsymbol{\psi}}$ that effectively identify the true unknown physics $\mathbf{f}_{\mathrm{unk}}$ along the true trajectory.

For the HRPINN, its task is to learn this same target function $\mathbf{f}_{\mathrm{unk}}(\mathbf{x}^*(t),t)$ with its network $\hat{\mathbf{f}}_{\boldsymbol{\phi}}$. The true trajectory defines a compact domain, so the UAT guarantees that for any $\epsilon_f > 0$, there exists an HRPINN network $\hat{\mathbf{f}}_{\boldsymbol{\phi}^*}$ such that $\|\hat{\mathbf{f}}_{\boldsymbol{\phi}^*}(\mathbf{x}^*(t), t) - \mathbf{f}_{\mathrm{unk}}(\mathbf{x}^*(t), t)\|_{L^\infty} < \epsilon_f$.

Let the HRPINN's trajectory error be $\mathbf{e}_k = \mathbf{h}_k - \mathbf{x}^*(t_k)$. We assume $\mathbf{e}_0 = \mathbf{0}$. Subtracting the true evolution equation from the HRPINN update rule \eqref{eq:appendix_hrpinn_update} yields an error recursion. Taking norms and applying the Lipschitz continuity of the dynamics (Assumption A3, with constant $L$) and the network approximation error $\epsilon_f$, we get:
\[
\|\mathbf{e}_{k+1}\| \le \|\mathbf{e}_k\|(1 + L\Delta t) + \Delta t \epsilon_f + \|\mathbf{LTE}_{k+1}\|.
\]
Substituting the bound for the LTE, we have:
\[
\|\mathbf{e}_{k+1}\| \le \|\mathbf{e}_k\|(1 + L\Delta t) + \Delta t \epsilon_f + C_{\mathrm{LTE}}\Delta t^{p+1},
\]
where we use the general order $p$ from Assumption A5 ($p=1$ for Forward Euler). Applying the discrete Grönwall inequality yields the global error bound for the HRPINN trajectory:
\[
\max_{k\le N}\|\mathbf{e}_k\| \;\le\; \frac{e^{LT}-1}{L}\,\big(\epsilon_f + C_{\mathrm{LTE}}\,\Delta t^{p}\big).
\]
This shows the HRPINN's trajectory converges to the true solution as the network approximation error $\epsilon_f \to 0$ and the time step $\Delta t \to 0$.

\paragraph{Part 2: (HRPINN $\Rightarrow$ PINN)}
Assume an HRPINN is trained such that its discrete trajectory $\{\mathbf{h}_k\}_{k=0}^N$ (produced by an order-$p$ integrator) approximates the true solution samples. The challenge is to show that a standard PINN can represent this solution, which requires constructing a continuous interpolant $\tilde{\mathbf{x}}(t)$ that also has a well-behaved derivative.

To do this, we construct a $C^1$ cubic Hermite interpolant $\tilde{\mathbf{x}}(t)$ over the intervals $[t_k, t_{k+1}]$. This interpolant is defined to match both the discrete states $\mathbf{h}_k, \mathbf{h}_{k+1}$ and suitable discrete derivative estimates (e.g., from the integrator's internal stages).

Standard interpolation error estimates (see, e.g., Hairer et al. \cite{hairer1993solving}) imply that for sufficiently smooth trajectories, this interpolant's error is bounded for both the value and its derivative:
\[
\|\tilde{\mathbf{x}}(t) - \mathbf{x}^*(t)\| = O(\Delta t^{p+1}),\qquad\|\dot{\tilde{\mathbf{x}}}(t) - \dot{\mathbf{x}}^*(t)\| = O(\Delta t^{p}),
\]
uniformly on $[0,T]$. We can now define two target functions for the standard PINN networks:
\begin{enumerate}
    \item \textbf{The Solution Target:} The continuous interpolant $\tilde{\mathbf{x}}(t)$ itself.
    \item \textbf{The Physics Target:} The residual function $r(t) \equiv \dot{\tilde{\mathbf{x}}}(t) - \mathbf{f}_{\mathrm{phys}}(\tilde{\mathbf{x}}(t), t)$.
\end{enumerate}
The error bounds above guarantee that $\tilde{\mathbf{x}}(t) \to \mathbf{x}^*(t)$ and $r(t) \to \mathbf{f}_{\mathrm{unk}}(\mathbf{x}^*(t), t)$ as $\Delta t \to 0$.

By the Universal Approximation Theorem (UAT), there exists:
\begin{itemize}
    \item A PINN network $\hat{\mathbf{x}}_{\boldsymbol{\theta}}$ that can approximate $\tilde{\mathbf{x}}(t)$ to arbitrary accuracy $\epsilon_x$.
    \item A PINN network $\hat{\mathbf{f}}_{\text{unk}, \boldsymbol{\psi}}$ that can approximate the continuous residual $r(t)$ to arbitrary accuracy $\epsilon_f$.
\end{itemize}
Therefore, the PINN's physics residual from Eq. \eqref{eq:appendix_std_pinn_residual} can be made arbitrarily small by decreasing the integrator error ($\Delta t$), the HRPINN network error, and the UAT approximation errors ($\epsilon_x, \epsilon_f$), demonstrating representational equivalence.

\end{proof}

\subsection{Discussion of Equivalence}
This theorem establishes that both frameworks possess the fundamental theoretical \textbf{capacity} to represent the same underlying ODE systems. It does not imply they will have similar training dynamics, practical performance on finite data, or ease of optimization. The choice between them is a practical one. HRPINN's structure provides a hard constraint for the known physics and the integrator, which can be advantageous for stability and interpretability. The standard PINN offers greater flexibility in defining residuals for diverse equation types (e.g., PDEs) directly in the loss function.

\vskip 0.2in
\bibliography{ref}

\end{document}